\documentclass{l4dc2025}

\title[Live LTL Progress Tracking]{Live LTL Progress Tracking: Towards Task-Based Exploration}
\usepackage{times}
\usepackage{algorithm2e}
\usepackage{textcomp}
\usepackage{xcolor}
\usepackage{tcolorbox}
\usepackage{url}
\usepackage{dsfont}
\usepackage{amssymb}
\usepackage{booktabs}

\DeclareMathOperator*{\argmax}{arg\max}

 \coltauthor{\Name{Noel C. Brindise} \Email{nbrindi2@illinois.edu}\\
   \Name{Cedric Langbort} \Email{langbort@illinois.edu}\\
  \Name{Melkior Ornik} \Email{mornik@illinois.edu}\\
  \addr Coordinated Science Laboratory, University of Illinois Urbana-Champaign, Urbana, USA}

\begin{document}

\def\ddefloop#1{\ifx\ddefloop#1\else\ddef{#1}\expandafter\ddefloop\fi}

    \def\ddef#1{\expandafter\def\csname c#1\endcsname{\ensuremath{\mathcal{#1}}}}
    \ddefloop ABCDEFGHIJKLMNOPQRSTUVWXYZ\ddefloop

    \def\ddef#1{\expandafter\def\csname s#1\endcsname{\ensuremath{\mathsf{#1}}}}
    \ddefloop ABCDEFGHIJKLMNOPQRSTUVWXYZ\ddefloop

    \def\ddef#1{\expandafter\def\csname b#1\endcsname{\ensuremath{\mathbb{#1}}}}
    \ddefloop ABCDEFGHIJKLMNOPQRSTUVWXYZ\ddefloop

\maketitle
\thispagestyle{plain}

\begin{abstract}%
Motivated by the challenge presented by non-Markovian objectives in reinforcement learning (RL), we present a novel framework to track and represent the progress of autonomous agents through complex, multi-stage tasks. Given a specification in finite linear temporal logic (LTL), the framework establishes a `tracking vector' which updates at each time step in a trajectory rollout. The values of the vector represent the status of the specification as the trajectory develops, assigning true, false, or `open' labels (where `open' is used for indeterminate cases). Applied to an LTL formula tree, the tracking vector can be used to encode detailed information about how a task is executed over a trajectory, providing a potential tool for new performance metrics, diverse exploration, and reward shaping. In this paper, we formally present the framework and algorithm, collectively named Live LTL Progress Tracking, give a simple working example, and demonstrate avenues for its integration into RL models. Future work will apply the framework to problems such as task-space exploration and diverse solution-finding in RL.
\end{abstract}


\section{Introduction}
\par In this work, we consider a common setting in reinforcement learning (RL) where goal specifications are known a priori, such as pre-determined tasks to complete or constraints to follow. We pay special attention to \textit{non-Markovian} tasks, which contain time or order dependencies. Though multi-stage, order-dependent tasks are quite common, a significant challenge arises when they are applied to the typical Markov decision process (MDP) model in RL, which can only capture Markovian dependencies. While it is possible to extend the system and explicitly encode past events in the state itself, this is prone to state-space explosion \citep{bonet2003labeled} and has prompted a wide variety of workarounds. The most prominent of these is the incorporation of replay buffers and memory \citep{liu2018effects}. A popular solution is long short-term memory (LSTM), which preserves information from recent time steps \citep{hochreiter1997long, 7508408}. Unfortunately, even state-of-the-art LSTMs suffer in both performance and computation time as the memory size increases, in part due to the inability of the model to reliably identify which past events are most relevant \citep{kandadi2025drawbacks}. In the special case where specifications are known a priori, it seems preferable to leverage this system knowledge to incorporate relevant events more directly into training. 

\par Variants of \textit{reward shaping} have incorporated memory in tailored ways to make non-Markovian tasks learnable, especially via recently-developed \textit{reward machines} and LTL \citep{icarte2018using}. In general, reward shaping carefully engineers functions to distribute rewards for improved feedback \citep{ibrahim2024comprehensive, ren2021learning, tang2024simplesumdelayedrewards}. In a reward machine (RM), the usual reward function is replaced by a finite state machine (FSM) whose internal states and transitions are structured according to the specification of interest, making it possible to assign rewards depending directly on a notion of progress \citep{icarte2022reward}. Several recent RM approaches consider specifications in LTL, a convenient syntax for time- and order-dependent expressions like tasks and constraints \citep{liao2020survey}, and shown that RMs on these objectives have asymptotic learning guarantees \citep{le2024reinforcement}. 

\par We present our framework, Live LTL Progress Tracking (LPT), and discuss its integration into existing RM methods; we also propose it as a starting point for less-explored capabilities in RL. In line with the suggestion by \cite{camacho2019ltl} that internal states of RMs be leveraged during learning, our framework decomposes LTL specifications into a novel multilayered structure, forming an enhanced, informative new tracking state. As shown in Figure \ref{fig:lpt-flowchart}, we first decompose a specification into a common tree structure. This tree is used by the algorithm to structure its tracking state. To perform live tracking, trajectory updates are fed into the algorithm, which produces a corresponding tracking state vector. This vector may serve as a reference to identify solutions which are meaningfully different from each other in terms of the task of interest, potentially providing a guiding diagnostic during exploration and a basis for diverse solution finding.

\begin{figure}[t]
\centering
\includegraphics[width=\columnwidth]{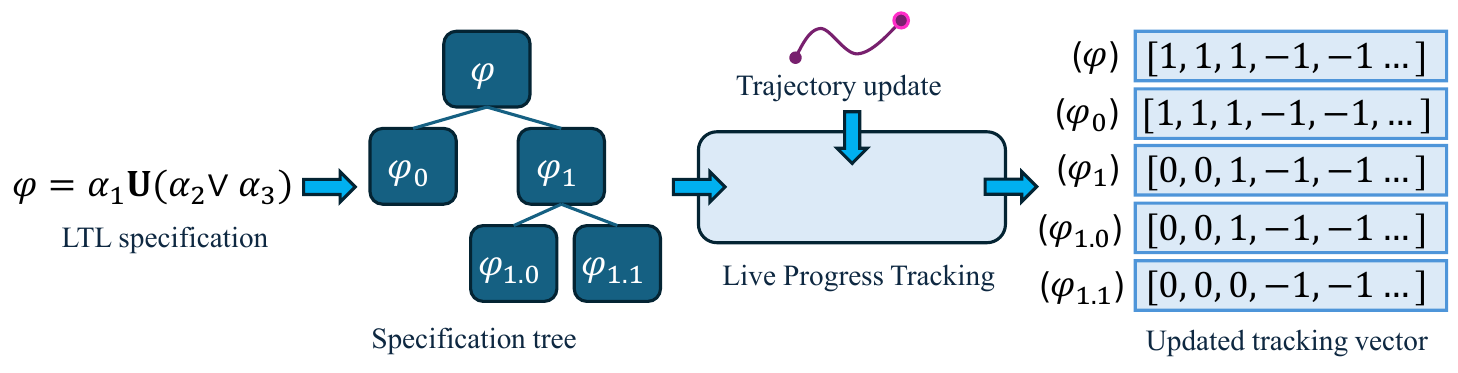}
\caption{A diagram of the Live Progress Tracking algorithm. The tracking vector output is updated at every trajectory time step, such that the vector values specify a \textit{true} (1), \textit{false} (0), or \textit{open} (-1) status for each tree node at each time step.}
\label{fig:lpt-flowchart}
\end{figure}

\par We intend this paper as a theoretical basis for future work and exclude exhaustive experimentation. We will present the algorithm structure and include a reference to a prototype Python implementation. We supplement the formalization with basic conceptual examples and discuss potential for integration and new capabilities in RL.

\subsection{Notation}
 \begin{table}[h!]
\centering
\begin{tabular}{cp{12.0 cm}}
$2^P$&The power set of $P$, i.e. all possible subsets of a set $P$.\\
\\
$f : X \rightarrow [Y\rightarrow Z]$&Function $f$ maps the set $X$ to a function space of all functions $Y\rightarrow Z$.\\
\\
$\phi_1\ |\ \phi_2$ & In a context-free grammar (CFG), the symbol $|$ is used to delineate \textit{alternatives}. In this notation, a rule $\varphi := \phi_1\ |\ \phi_2$ specifies that $\varphi$ may take the form $\varphi=\phi_1$ or $\varphi=\phi_2$. For further discussion, see \cite{1056813} and \cite{hopcroft2001introduction}.\\

\end{tabular}
\end{table}

\FloatBarrier

\section{Previous Work}

\par In this section, we situate our work more thoroughly within the literature. Firstly, a variety of methods exist to seek optimal policies for non-Markovian tasks in linear temporal logic (LTL). For stochastic systems, a common objective is to maximize the probability that a specification is satisfied by a strategy. In RL, this strategy is a \textit{policy}, a function prescribing the choice of actions taken by the agent \citep{camacho2017non}. LTL synthesis, the most-explored framing for LTL specifications on transition systems, involves a two-step process. First, the LTL specification is converted into an automaton, such as a limit-deterministic B\"uchi automaton (LDBA). The original system is likewise converted, and the two are combined into a single system over a product state space \citep{baier2008principles, wang2020continuous}. Next, an optimal policy is found by attempting to maximize the probability that good (``accepting'') states of the automaton are reached via dynamic programming \citep{fu2014probably}. 
\par Problematically for the first step, automaton generation is computationally intensive \citep{kretinsky2025semmlenhancing}. Moreover, the guarantees obtained in the maximization step are severely conditioned on sufficient sampling \citep{ding2011ltl, 6702421}. It is thus perhaps unsurprising that practical examples are typically limited to low-complexity tasks, such as simple reach-avoid problems \citep{cai2023safe, 8431181}. 

\par Partly in response to the challenges of classic LTL synthesis, some approaches in RL have instead shaped a reward function to correspond to progress through the specification. In many cases, these introduce a distance measure to calculate how far the current state is from an accepting state in terms of transitions; actions which take the agent `in the right direction' are rewarded \citep{9981759, kantaros2022accelerated, kwon2025adaptivereward}, \cite{shah2025ltl}. For complicated tasks, many approaches seek modularity: breaking TL specifications into simpler subtasks to be learned individually and then applied to multi-step or multi-objective problems \citep{yuan2019modular}. Reward shaping has also been implemented by approximating a value function on an LDBA \citep{bagatella2025directed}.

\par Sometimes in combination with the aforementioned strategies, LTL reward machines have begun to gain popularity. Similarly to LTL synthesis, these machines typically begin by constructing an automaton from an LTL specification which is then used to assign rewards. These RMs have been developed for time-discounted LTL \citep{alur2023policy}, modularized versions of LTL tasks (``sequential'' LTL) \citep{zheng2022lifelong}, and finite LTL \citep{camacho2018non}. However, RMs structured in this way suffer from the same complexity hurdle as LTL synthesis during automaton generation.

\par To track a more explicit notion of progress, some approaches incorporate \textit{LTL progression}, first proposed by \cite{BACCHUS2000123}. LTL progression is a method to update a specification dynamically as a trajectory rolls out, simplifying the LTL specification as its individual requirements are satisfied. In this way, progression tracks the requirements `to go' from a current timestep, i.e., which requirements remain to be met at that point. For example, \cite{pmlr-v139-vaezipoor21a} encode a specification as a Graph Neural Network (GNN) and assigns rewards based on LTL progression. Similarly to an application we will recommend in this paper, some methods use progression to identify subtasks and learn value functions for each of them \citep{toro2018teaching}. 

\par We avoid LTL progression because of the way it expresses progress, namely as a dynamic specification which compiles remaining requirements into one running `to-do.' As is the case in the automata of LTL synthesis, the single representation obscures many of the finer details of how a specification is met. When applied to the exploration problem, such as by \cite{hasanbeig2020deep}, we propose that our framework will help to identify a wider variety of behaviors. We will explore the distinction in our example scenario.

\section{Preliminaries}

\subsection{Reinforcement Learning and Reward Machines}
\par We primarily propose the use of our LPT framework for RL. RL is a paradigm in machine learning in which agents learn behaviors using rewards or costs as feedback. The agent and its environment are generally expressed together as a Markov Decision Process (MDP). In our case, we will introduce a \textit{labeling function} to the base MDP.
\begin{definition}[Labeled Markov Decision Process]
    A labeled MDP is a tuple
    \begin{equation}
        \mathcal{M}=\langle S, A, T, R, P, \mathcal{L}\rangle,
    \end{equation}
    where $S$ is a finite set of discrete states, $A$ is a finite set of actions, $T: S\times A \times S \rightarrow \mathbb{R}$ is a transition function expressing the probability of state transitions given a choice of action, and $R: S\times A \rightarrow \mathbb{R}$ is a reward function. $P$ is a finite set of labels, and $\mathcal{L} : S \rightarrow 2^P$ is a labeling function assigning a set $L\subseteq P$ to each state in $S$. \label{def:MDP}
\end{definition}
An agent prescribes a \textit{policy} $\pi: S \rightarrow A$ to map states to actions, aiming to optimize its \textit{expected cumulative reward} on the MDP. This may incorporate a discount factor $\gamma\in [0,1]$:
\begin{equation}
    V^\pi(s_t) = \mathbb{E}_\pi \big[\sum_{k=0}^\infty\gamma^k R(s_{t+k},a_{t+k}) \big]
\end{equation}
where $\mathbb{E}_\pi$ denotes the expectation given that $a_t=\pi(s_t)$ and transitions $(s_t,a_t,s_{t+1})$ occur with probability $T(s_t,a_t, s_{t+1})$. The optimal policy is $\pi^*(s) = \argmax_\pi  V^\pi(s)$.  

\par Recently, some RL approaches have replaced or supplemented the simple reward function with a \textit{reward machine}. We consider a simplified version adapted from \cite{icarte2018using}: 

\begin{definition}[Reward Machine]\label{def:rewardMachine}
Let $\mathcal{M}=\langle S, A, T, R, P, \mathcal{L}\rangle$. A reward machine for $\mathcal{M}$ is
\begin{equation}
    \mathcal{R}_{M} = \langle U, u_0, F, \delta_u, \delta_r\rangle
\end{equation}

where $U$ is a finite set of internal states, $u_0\in U$ is an initial state, F is a finite set of terminal states such that $U \cap F = \emptyset$, $\delta_u : U \times 2^P \rightarrow U\cup F$ is the internal transition function, and $\delta_r : U \rightarrow [S\times A\rightarrow\mathbb{R}]$ is the reward function.  

\end{definition}
Note that the function $\delta_r$ produces functions $S\times A\rightarrow \mathbb{R}$ as its output. The machine's structure makes it possible to represent rewards which are non-Markovian on $\mathcal{M}$ by allowing the reward function to vary depending on an internal state. 

\subsection{System Traces and Linear Temporal Logic}
 From its labeling function, any system run over the MDP of Definition \ref{def:MDP} produces a sequence of sets of labels over time. This sequence, called a system trace, serves as a description of the trajectory. We consider systems producing finite traces:

\begin{definition}[Finite Trace]
Let $s_t$ denote the state in a trajectory at time $t$. Denote its corresponding label set $\mathcal{L}(s_t)$ by $L_t$. A finite trace is the sequence $\rho = (L_{T_0},...,L_{T_f})$ produced by system trajectory $(s_{T_0},...,s_{T_f})$ over discrete time interval $\{T_0,...,T_f\}$.
\end{definition}

\begin{definition}[Finite Trace Suffix]
    The trace suffix $\rho^{t...}=(L_t,...,L_{T_f})$ is the part of the full trace $\rho$ which begins at $t$, where $T_0\leq t \leq T_f$. 
\end{definition}

\par Our method analyzes traces using LTL \citep{pnueli1977temporal}. LTL is a fixture in formal methods and rule inference, where it is used to express \textit{constraints,} such as safety requirements or operational limits, and \textit{tasks,} such as goals and processes \citep{camacho2019learning}. Formulas are expressed over a vocabulary of \textit{atomic propositions}, in our case the set $P$ of labels from $\mathcal{M}$. We specifically adopt finite LTL (LTL$_f$), which adapts the standard (infinite-trace) LTL satisfaction conditions to apply to finite sequences \citep{de2013linear, fionda2018ltl}. A general formula in LTL and LTL$_f$ can be expressed as a context-free grammar as follows:
\begin{equation}\label{eq:cfg}
    \varphi := true\ |\ \alpha \ |\ \varphi_1 \wedge \varphi_2 \ |\ \neg \varphi_1 \ |\ \mathbf{X}\varphi_1 \ |\ \varphi_1 \mathbf{U} \varphi_2 \quad \text{where $\alpha\in P$.}
\end{equation}

 \par Given finite suffix $\rho^{t...} = (L_{t},...,L_{T_f})$, the truth of $\varphi$ on $\rho^{t...}$ is determined by Definition \ref{FiniteTraceDef}.

\begin{definition}[Formula Evaluation on Finite Traces]\label{FiniteTraceDef}
    LTL formula $\varphi$ is true on $\rho^{t...} = (L_{t},...,L_{T_f})$, denoted $\rho^{t...}\models\varphi$, if $T_0\leq t \leq T_f$ \textbf{and}:
        \par$\bullet$ $\rho^{t...} \models \alpha$ where $\alpha\in P\quad$ iff $\quad\alpha\in L_{T_0}$
        \par$\bullet$ $\rho^{t...}\models \neg \varphi\quad$ iff $\quad\rho^{t...}\not\models\varphi$
        \par$\bullet$ $\rho^{t...} \models \varphi_1 \wedge \varphi_2\quad$ iff $\quad\rho^{t...}\models\varphi_1$ and $\rho^{t...}\models\varphi_2$
        \par and for the temporal operators \textup{Next} $\mathbf{X}$ and \textup{Until} $\mathbf{U}$, 
    \par$\bullet$ $\rho^{t...} \models \mathbf{X}\varphi\quad$ iff $\quad \rho^{t+1...}\models\varphi$ \textbf{and} $t<T_f$
    \par$\bullet$ $\rho^{t...}\models\varphi_1\mathbf{U}\varphi_2\quad$ iff $\quad \exists i\geq t$ s.t. $\rho^{i...}\models\varphi_2$ and $\rho^{k...} \models \varphi_1$ $\forall k$, $t\leq k < i$ \textbf{and} $i\leq T_f$.
\end{definition}
The condition $\rho\models\varphi$, read as ``$\varphi$ holds on $\rho$,'' means that the trace $\rho$ satisfies the conditions set by formula $\varphi$. For example, if $\varphi$ specifies a task to be completed, $\rho\models\varphi$ means that the agent behavior captured by $\rho$ successfully completed the task. The notation $\rho\not\models\varphi$ denotes the negation of this.
\par For convenience, additional LTL operators $\vee$ and $\mathbf{G,F,R,M,W}$ are typically defined as well, though they can be constructed directly from $\wedge,\neg, \mathbf{X,U}$ \citep{pnueli1977temporal}. Intuitive interpretations of each operator are given in Table \ref{tab:LTLoperators}. 
\par Now we define formula \textit{arguments} to decompose formulas into their constituent parts:

 \begin{table}[t]
\centering
\begin{tabular}{cp{12.0 cm}}
\toprule
 \textbf{Operator}&\textbf{Truth Condition (when $\rho\models \varphi$)}\\
 \midrule
\textbf{G}$\varphi$&\textbf{Global}: $\varphi$ is always true\\
\midrule
\textbf{F}$\varphi$&\textbf{Eventual}: $\varphi_1$ is eventually true\\
\midrule
\textbf{X}$\varphi$&\textbf{Next}: $\varphi_1$ must be true at the next time step\\
\midrule
$\varphi_1$\textbf{U}$\varphi_2$&\textbf{Until}: $\varphi_1$ remains true until $\varphi_2$ is true (the latter must occur at some $t$)\\
\midrule
$\varphi_1$\textbf{W}$\varphi_2$&\textbf{Weak until}: $\varphi_1$ must remain true (1) always, or (2) until $\varphi_2$ becomes true\\
\midrule
$\varphi_1$\textbf{R}$\varphi_2$&\textbf{Release}: $\varphi_2$ must remain true (1) always, or (2) until (and including) the time step when $\varphi_1$ becomes true\\
\midrule
$\varphi_1$\textbf{M}$\varphi_2$&\textbf{Strong release}: True only if Condition (2) for Release is met\\
\midrule
$\neg\varphi$&\textbf{Negation}: $\varphi$ must be false \\
\midrule
$\varphi_1\wedge\varphi_2$&\textbf{Conjunction}: both $\varphi_1$ and $\varphi_2$ must be true\\
\midrule
$\varphi_1\vee\varphi_2$&\textbf{Disjunction}: $\varphi_1$ or $\varphi_2$ or both must be true\\
\midrule
$\varphi_1\rightarrow\varphi_2$&\textbf{Implication}: if $\varphi_1$ is true, $\varphi_2$ must also be true\\
\bottomrule\\
\end{tabular}
\caption{\label{tab:LTLoperators} Intuitive descriptions of LTL operators.}
\end{table}

\begin{definition}[Arguments of an LTL formula]\label{def:formula_args}
Consider the LTL order of operations: (1) grouping symbols; (2) $\neg, \mathbf{X}$, and other unary operators; (3) $\mathbf{U}$ and other temporal binary operators; and (4) $\vee, \wedge, \rightarrow$. For LTL formula $\varphi$, all $\varphi_j$ bound by the weakest operator are the \textit{arguments} of $\varphi$.
\end{definition}

\par Recalling our purpose, we are interested in a framework to track an agent's progress in terms of specifications. This leads us to the proposed framework, which we will define in terms of LTL formulas and their arguments.

\section{LTL Live Progress Tracking}
\subsection{Overview}
LTL Live Progress Tracking (LPT) is a novel framework that uses LTL specifications to describe and track behaviors, such as an agent completing steps in a task. This is done by mapping agent trajectories to traces and providing the updated traces to the framework as the trajectory progresses. The framework ultimately encodes the trace as a `signature:' a vector which serves as a unique identifier for behavior patterns. 

\par In short, LTP signatures are generated from a ternary tracking vector which contains a value for each time step in the trace. The values correspond to a truth status for each argument of the specification (true, false, open), and they update continuously as the agent trace rolls out (hence the \textit{live} descriptor). Intuitively, an argument is assigned \textit{true} at $t$ only if its formula must logically hold on any suffix from $t$, given what has already occurred. It is likewise \textit{false} only if the formula must logically be false. The \textit{open} assignment means that the algorithm has not identified conditions for \textit{true} nor \textit{false}. (We intentionally avoid a stronger statement; we reserve the discussion of this for the Definitions below.)
\par The example below gives a brief demonstration of how such signatures can describe system traces. 

\begin{example}
    Consider an agent which should eventually collect two keys, Key A and Key B, represented by formula 
    \begin{equation*}
        \varphi =\mathbf{F}\texttt{keyA}\wedge\mathbf{F}\texttt{keyB}.
    \end{equation*} 
\end{example}
\par We begin by considering the formula tree for $\varphi$, which has five nodes: 
\begin{equation}\label{eq:treenodes}
    \mathbf{F}\texttt{keyA}\wedge\mathbf{F}\texttt{keyB}\quad \mathbf{F}\texttt{keyA}\quad \mathbf{F}\texttt{keyB}\quad\texttt{keyA}\quad\texttt{keyB}
\end{equation} The LPT framework establishes a tracking vector for each of these nodes, with each vector containing one entry per time step. The signature is simply a concatenation of these vectors with consecutive repeated values merged, e.g. $[0,0,1,-1,-1]$ becomes $[0,1,-1]$, since this preserves the order of events while removing excess information about precise timing. For a fixed formula, trajectories can then be compared behaviorally in terms of their signatures.

\begin{figure}[h]
\centering
\includegraphics[width=.9\columnwidth]{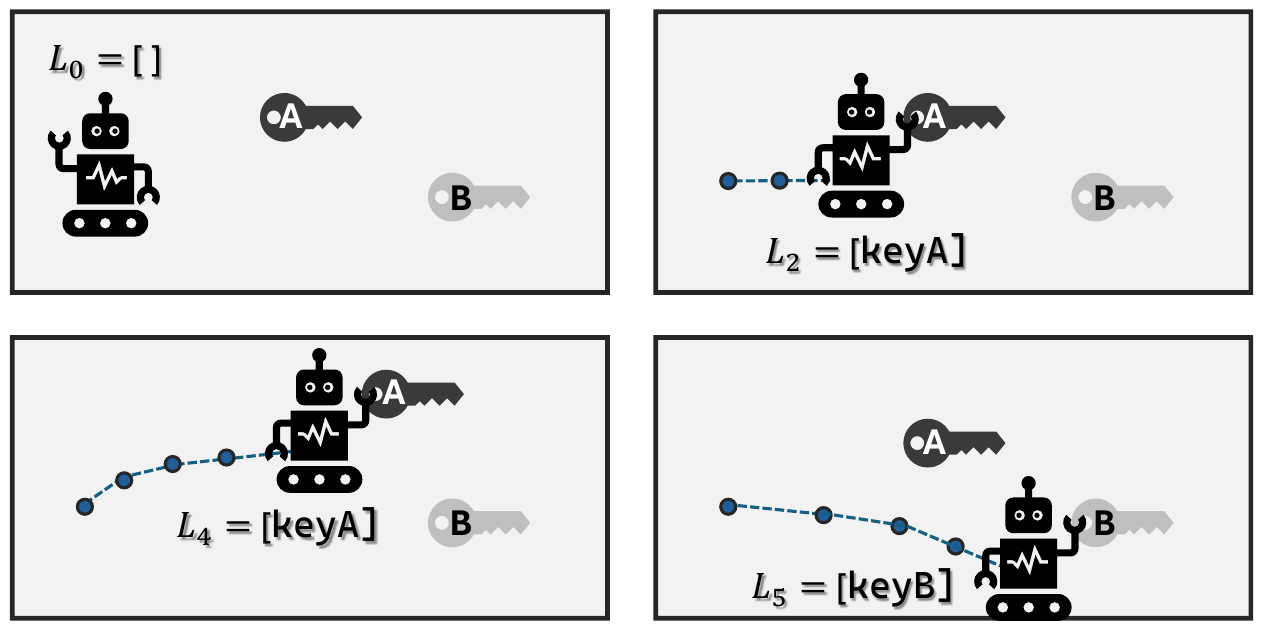}
\caption{Snapshots from a key collection example.}
\label{fig:signature_ex}
\end{figure}

\par Consider Figure \ref{fig:signature_ex}. On the upper left, an agent begins at $t=0$ with trace $\rho^{0...0}=(\{\})$. $(\texttt{keyA})$ and $(\texttt{keyB})$ are false. It is unknown whether the agent will obtain either key in the future, so $(\mathbf{F}\texttt{keyA})$, $(\mathbf{F}\texttt{keyB})$ and their conjuction are open. Listing the nodes in the same order as \eqref{eq:treenodes}, the signature is 
\begin{equation*}
[[-1],[-1],[-1],[0],[0]].
\end{equation*}

\par On the upper right, the agent is at $t=2$ with $\rho^{0...2}=(\{\},\{\},\{\texttt{keyA}\})$. Now that key A has been obtained, $(\mathbf{F}\texttt{keyA})$ is conclusively true for $\rho^{0...},\rho^{1...},$ and $\rho^{2...}$. Key B is not yet acquired, so $(\texttt{keyB})$ is false for $t=0,1,2$. However, it may be acquired in the future, so $(\mathbf{F}\texttt{keyB})$ remains open. Because of this, $(\mathbf{F}\texttt{keyA}\wedge \mathbf{F}\texttt{keyB})$ is open as well. In all, the signature is
 \begin{equation*}
     [[-1],[1],[-1],[0,1],[0]].
 \end{equation*}

\par The bottom left shows an \textbf{alternative} path, where key A is obtained at $t=4$. The new trace is $\rho^{0...4}=(\{\},\{\},\{\},\{\},\{\texttt{keyA}\})$. The timing is different, but the behavior is the same as before and has an identical signature. 

\par Finally, the bottom right shows an alternative where key B is obtained \textbf{instead of} key A. The trace is $\rho^{0...5}=(\{\},\{\},\{\},\{\},\{\},\{\texttt{keyB}\})$. Following the same reasoning as previously, the signature is
\begin{equation*}
    [[-1],[-1],[1],[0],[0,1]]
\end{equation*}
reflecting the change in behavior.

\par The signatures of LPT track progress for every node in a rule tree, giving higher-fidelity information than the automata in LTL synthesis (discussed in section \ref{automatoncompare}). We propose that this gives potential advantages for exploration and diverse solution-finding. We will now define the framework as in \cite{brindisephd} and give an algorithm which generates tracking data in adherence with the framework. We then demonstrate an application of LPT and present a reward machine structure for use in RL.

\subsection{Definitions}

LPT tracks progress at a selected `current' time step $t'$ during a trajectory rollout. The trajectory up to $t'$ is described by finite trace
\begin{equation}
    \rho^{t_0...t'} := (L_{t_0},...,L_{t'}).
\end{equation}
For a rollout which terminates at a future time $T_f > t'$, we call the future sequence of unknown values a \textit{continuation:}
\begin{definition}[Trace Continuation]\label{def:continuation}
    Consider $\rho^{t...t'} = (L_t,...,L_{t'})$, where $L_t,...,L_{t'}$ are known and remain fixed. Then a \textit{continuation} of $\rho^{t...t'}$ refers to any trace created by appending an arbitrary label set sequence to $\rho^{t_0...t'}$, i.e.
    \begin{equation}
        \rho^{t...t^+} = (L_t,...,L_{t'}, L^+_{t'+1},...,L^+_{t^+})
    \end{equation}
    where each $L^+$ is any set of labels $L\subseteq P$.
\end{definition}
Note that this definition does not require $\rho^{t...t^+}$ to represent a feasible trajectory on $\mathcal{M}$.
\par The goal of LPT is to assign tracking values \textit{true} (h), \textit{false} (v), and \textit{open} (o) to each formula argument appropriately at each time step given the known information in $\rho^{t...t'}$. A value assignment is considered \textit{valid} if it satisfies the following properties:

\begin{definition}[Tracking value properties]\label{def:trackingprops}
Consider an LTL formula $\varphi$ on $\rho^{t...t'}$ with continuations $\rho^{t...t^+}$. Then a value assignment of (h),(v), or (o) is \textbf{valid} if:
\par $\bullet$ \textup{true (h)}: $\rho^{t...t^+}\models\varphi$ for all $t^+\geq t'$.
\par $\bullet$ \textup{false (v)}: $\rho^{t...t^+}\not\models\varphi$ for all $t^+\geq t'$.
\par $\bullet$ \textup{open (o)}: any of these conditions may or may not be met.
\end{definition}

\par Definition \ref{def:trackingprops} deliberately leaves flexibility in the assignment of (o), using it as a catch-all for cases where $\rho^{t_0...t^+}\stackrel{?}{\models}\varphi$ is perhaps determinate but not resolved by the algorithm. For example, consider the formula $\varphi=\varphi_0 \wedge\neg\varphi_0$. If the truth of $\rho\models\varphi_0$ is unknown (o), the LPT algorithm as defined below will assign $\varphi$ as (o), despite the fact that $\rho\not\models\varphi$ by logical deduction. Since LPT is focused on tracking rather than formal verification, this is an acceptable limitation.

\par Value assignments are organized for LTL arguments as a \textit{tracking value vector}, which depends on current time step $t'$:
\clearpage
\begin{definition}[Tracking value vector]\label{def:lpt_trackingvalvec}
    \sloppy Consider $\varphi$ and finite traces in discrete time $\rho^{T_0...}$ with final time $T_f$. Then a tracking value vector for any argument of $\varphi$ can be expressed as
        ${\vec{\xi} = (\xi[{T_0}], ...,\xi[{t'}],..., \xi[{T_f}])^T}$,
     where $t'$ is a current update time with $t'\in\{T_0,...,T_f\}$. For all entries $\xi[t]$, 
    \begin{equation}
        \xi[t] = 
        \begin{cases}
            1 & \varphi \text{ has value (h) on }\rho^{t...t'}\\
            0 & \varphi \text{ has value (v) on }\rho^{t...t'}\\
            -1 & \text{otherwise}
        \end{cases}
    \end{equation}
\end{definition}

\begin{definition}[Validity of tracking value vector]\label{def:validtvv}
    A tracking value vector $\vec{\xi}$ is called valid for LTL formula $\varphi$ on $\rho^{T_0...}$ if and only if all entries $\xi[t]$ correspond to a valid assignment of (h), (v), or (o) by Definition \ref{def:trackingprops}.
\end{definition}

We now present an algorithm to calculate tracking value vectors which provably satisfy Definition \ref{def:validtvv}. Such an algorithm will be able to provide module-by-module tracking information as we have described for system trajectories.

\subsection{Algorithm}
\subsubsection{Definitions and Guarantees}
For any specification, the LPT algorithm begins by constructing a formula tree. As in a standard LTL tree, the top-level node contains the original formula, which has child node(s) containing its arguments, and so on, such that the leaves of the tree are atomic formulas (labels). Each node is additionally assigned a \textit{type} to identify its top-level LTL operator (e.g. $\mathbf{F},\wedge,\mathbf{U},...$). The algorithm then processes a current trajectory for the entire tree from the bottom, starting from the tree leaves ($\varphi:=\alpha$) and progressing upward. For each node, it assigns a tracking value for the argument in question based on its formula type and the tracking values of its children. Once the algorithm finishes at the top level, tracking value vectors have been populated for every node in the tree.

\begin{tcolorbox}[colback=gray!10!white,colframe=gray!75!black,title=Live LTL Progress Tracking Algorithm]
\begin{enumerate}
\item Construct tree for LTL rule $\varphi$. For each node, store the top-level LTL operator as the node \textit{type} $\theta$.
\item Find $\vec{\xi}$ of each node for all $\rho^{t_0...t'}$ such that $t_0\leq t'$ and ${t_0,t'\in\{T_0,...,T_f\}}$:
\par \textbf{Initialize} $\vec{\xi} = (-1,...,-1)$ for all nodes.
\par \textbf{Set} $t'=T_0$. \textbf{While} $t'\leq T_f$:
\begin{enumerate}
\item For each distinct label $\alpha$ appearing in any leaf, pass $\rho^{T_0...t'}$, ${\varphi=\alpha}$ into the $AP$ module. Store outputs $\vec{\xi}$ for all leaves containing this $\alpha$.
\item For leaf $\varphi_{...x.y}$, check parent $\varphi_{...x}$. If $\vec{\xi}$ is already stored for all children of $\varphi_{...x}$, proceed to (c); otherwise, move to next leaf and repeat (b).
\item Identify $\theta$ of $\varphi_{...x}$. Pass $\varphi_{...x}$, $\rho^{T_0...t'}$ into $\theta$ module; store output $\vec{\xi}$ for $\varphi_{...x}$.
\item Check parent of $\varphi_{...x}$. If it exists and all its children are instantiated, repeat step c) for this parent. Otherwise, move to next leaf and begin again from step (b).
\item \textbf{If} $t'=T_f$, perform \textbf{terminal evaluation} for all tree nodes $\phi$ of $\varphi$: for all $t\in\{T_0,...,T_f\}$, if $\xi[t]=-1$, set
    \begin{align}\label{eq:terminal_eval}
        \xi[t]=
        \begin{cases}
            1&\rho^{t...T_f}\models\phi\\
            0&\rho^{t...T_f}\not\models\phi
        \end{cases}
    \end{align}
\item Increment $t'$.
\end{enumerate}
\end{enumerate}
\end{tcolorbox}

The modules for LPT are given in Tables \ref{tab:logmodules_lpt} and \ref{tab:tempmodules_lpt}. Intuitively, each module determines an assignment (o), (h), or (v) based on the LTL definitions for its node type. For example, a node containing $\varphi=\varphi_x \vee \varphi_y$ will determine whether $\rho\models\varphi$ by checking the output of its child nodes, which contain the current tracking values for $\rho\models\varphi_x$ and $\rho\models\varphi_y$. The module is designed to assign for $\varphi$ based on the definition of $\vee$ (`or'). Intuitively, this will be true (h) if either argument is conclusively true (h), false (v) if both arguments are conclusively false (v), and open (o) otherwise.

\par To formally guarantee that every module makes logically-sound assignments based on the LTL definitions, we will prove Theorem \ref{correctnessLPT}.

\begin{theorem}[LPT Algorithm Output]\label{correctnessLPT}
Consider an LTL formula $\varphi$ and trace $\rho^{T_0...t'}$ which updates incrementally ($t'=T_0,...,T_f$). For inputs at each update $t'$, the LPT algorithm always assigns valid tracking value vectors $\vec{\xi}$ by Definition \ref{def:validtvv} for all nodes of $\varphi$.
\end{theorem}

The full proof of this theorem is in Appendix \ref{appendixA}. In brief, the proof follows an inductive logic:
\begin{enumerate}
    \item Show that $\vec{\xi}$ are initially valid tracking value vectors.
    \item When performing an update at $t'$, assume that all $\vec{\xi}_{t'-1}$ are valid.
    \item Prove that, given a valid $\vec{\xi}_{t'-1}$, the algorithm generates a $\vec{\xi}_{t'}$ which is a valid tracking value vector.
\end{enumerate}
This is accomplished in part by establishing properties of the evolution of correct $\vec{\xi}$ and examining the logic for each operator in a module-by-module way.
\FloatBarrier

\begin{table}[h]
\centering
\begin{tabular}{lp{11 cm}}
\toprule
\textbf{Not} $(neg)$& LTL: $\varphi=\neg\varphi_1$\\
\midrule
Initialize: & Load $\vec{\xi}$ for $\varphi$ and $\vec{\xi}_1$ for child node $\varphi_1$.\\
For $t=T_0,...,t'$:&If $\xi_1[t]=0$:\quad Set $\xi[t]= 1$.\\
    &Else if $\xi_1[t]=1$:\quad Set $\xi[{t}] = 0$.\\
\midrule
\textbf{Or} $(or)$&LTL: $\varphi=\varphi_1\vee\varphi_2$\\
\midrule
Initialize: & Load $\vec{\xi}$ for $\varphi$ and $\vec{\xi}_1,\vec{\xi}_2$ for child nodes $\varphi_1,\varphi_2$.\\
For $t=T_0,...,t'$:&If $\xi_1[t]=1$ or $\xi_2[t]=1$:\quad Set $\xi[t] = 1$. \\
    &Else if $\xi_1[t]=0$ and $\xi_2[t]=0$:\quad Set $\xi[t] = 0$.\\
\midrule
\textbf{And}  $(and)$&LTL: $\varphi=\varphi_1\wedge\varphi_2$\\
\midrule
Initialize: & Load $\vec{\xi}$ for $\varphi$ and $\vec{\xi}_1,\vec{\xi}_2$ for child nodes $\varphi_1,\varphi_2$.\\
For $t=T_0,...,t'$:&If $\xi_1[t]=1$ and $\xi_2[t]=1$:\quad Set $\xi[t] = 1$. \\
    &Else if $\xi_1[t]=0$ or $\xi_2[t]=0$:\quad Set $\xi[t] = 0$.\\
\midrule
\textbf{Implication}  $(\rightarrow)$&LTL: $\varphi=\varphi_1\rightarrow\varphi_2$\\
\midrule
Initialize: & Load $\vec{\xi}$ for $\varphi$ and $\vec{\xi}_1,\vec{\xi}_2$ for child nodes $\varphi_1,\varphi_2$.\\
For $t=T_0,...,t'$:&If $\xi_1[t]=0$ or $\xi_2[t]=1$:\quad Set $\xi[t] = 1$. \\
    &Else if $\xi_1[t]=1$ and $\xi_2[t]=0$:\quad Set $\xi[t] = 0$.\\
\bottomrule\\
\end{tabular}
\caption{\label{tab:logmodules_lpt} Logical operator module definitions.}
\end{table}

\begin{table}[t]
\centering
\begin{tabular}{lp{11.0 cm}}
\toprule
\textbf{Atomic}  $(AP)$ & LTL: $\varphi=\alpha$\\
 \midrule
Initialize: & Load $\vec{\xi}$ for $\varphi$.\\
If $\alpha\in L_{t'}$:&Set $\xi[t']= 1$.\\
Else: &Set $\xi[{t'}] = 0$.\\
 \midrule
\textbf{Next}  $(X)$&LTL: $\varphi=\mathbf{X}\varphi_1$\\
\midrule
Initialize: & Load $\vec{\xi}$ for $\varphi$ and $\vec{\xi}_1$ for child node $\varphi_1$.\\
For $t=T_0,...,t'$:&If $\xi_1[t] = 1$, $t>T_0$:\quad Set $\xi[t-1] = 1$.\\
&Else if $\xi_1[t] = 0$, $t>T_0$\textbf{ or } $t'=T_f$: \quad Set $\xi[t-1] = 0$.\\
\midrule
\textbf{Eventual} $(F)$&LTL: $\varphi=\mathbf{F}\varphi_1$\\
\midrule
Initialize: & Load $\vec{\xi}$ for $\varphi$ and $\vec{\xi}_1$ for child node $\varphi_1$.\\
For $t=T_0,...,t'$:& If $\xi_1[t]=1$:\quad For all $t_0$ s.t. $T_0\leq t_0 \leq t$: set $\xi[t_0]= 1$.\\
\midrule
\textbf{Global} $(G)$&LTL: $\varphi=\mathbf{G}\varphi_1$\\
\midrule
Initialize: & Load $\vec{\xi}$ for $\varphi$ and $\vec{\xi}_1$ for child node $\varphi_1$. Set $t=t'$.\\
While $t\geq T_0$:& If $\xi_1[t]=0$:\quad For all $t_0$ s.t. $T_0\leq t_0 \leq t$: set $\xi[t_0]= 0$. Set $t=T_0-1$.\\
& Else: set $t=t-1$.\\
 \midrule
 
 \textbf{Until} $(U)$&LTL: $\varphi=\varphi_1\mathbf{U}\varphi_2$\\
 \midrule
 & See Algorithm \ref{alg:untilLPT}.\\

 \midrule
 \textbf{W. until} $(W)$&LTL: $\varphi=\varphi_1\mathbf{W}\varphi_2$\\
 \midrule
 &Same as $\mathbf{U}$.
\\
\midrule

 \textbf{S. release} $(M)$&LTL: $\varphi=\varphi_1\mathbf{M}\varphi_2$\\
 \midrule
 & See Algorithm \ref{alg:sreleaseLPT}.\\
\midrule
 \textbf{Release} $(R)$&LTL: $\varphi=\varphi_1\mathbf{R}\varphi_2$\\
 \midrule
 &Same as $\mathbf{M}$.\\
\bottomrule\\
\end{tabular}
\caption{\label{tab:tempmodules_lpt} Temporal operator module definitions.}
\end{table}


\begin{algorithm}[t]
\SetAlgoLined
\SetKwFunction{getUpdates}{getUpdates}

\nl\textbf{Initialize:} Load $\vec{\xi}$ for $\varphi$ and $\vec{\xi}_1,\vec{\xi}_2$ for child nodes $\varphi_1,\varphi_2$.\\
\nl Set  $t_0,t= T_0$.

\nl \While{$t\leq t'$}{
    \nl \uIf{$\xi_1[t]=0$ and $\xi_2[t]=0$}{
        \nl Set $\xi[t] = 0$\\
        \nl Set $t = t+1$\\
        \nl Set $t_0 = t$
    }
    \nl \uElseIf{$\xi_2[t] = 1$}{
        \nl Set $\xi[t_0],...,\xi[t] = 1$\\
        \nl Set $t = t+1$\\
        \nl Set $t_0 = t$
    }
    \nl \uElseIf{$\xi_1[t] = 0$ or $\xi_1[t] = -1$}{
        \nl Set $t = t+1$\\
        \nl Set $t_0 = t$
    }
    \nl\uElse{
        \nl Set $t = t+1$
    }
}
\caption{Until/Weak Until Modules for LPT.}\label{alg:untilLPT}
\end{algorithm}


\begin{algorithm}[t]
\SetAlgoLined
\SetKwFunction{getUpdates}{getUpdates}

\nl\textbf{Initialize:} Load $\vec{\xi}$ for $\varphi$ and $\vec{\xi}_1,\vec{\xi}_2$ for child nodes $\varphi_1,\varphi_2$.\\
\nl Set  $t_0,t= T_0$.

\nl \While{$t\leq t'$}{
    \nl \uIf{$\xi_2[t]=0$}{
        \nl Set $\xi[t] = 0$\\
        \nl Set $t = t+1$\\
        \nl Set $t_0 = t$
    }
    \nl \uElseIf{$\xi_1[t] = 1$ and $\xi_2[t] = 1$}{
        \nl Set $\xi[t_0],...,\xi[t] = 1$\\
        \nl Set $t = t+1$\\
        \nl Set $t_0 = t$
    }
    \nl \uElseIf{$\xi_2[t] = -1$}{
        \nl Set $t = t+1$\\
        \nl Set $t_0 = t$
    }
    \nl\uElse{
        \nl Set $t = t+1$
    }
}
\caption{Release/Strong Release Module for LPT.}\label{alg:sreleaseLPT}
\end{algorithm}

\FloatBarrier

\subsubsection{Complexity}
The computational complexity of LPT is dependent on the length of the trajectory, the depth of the specification, and the specification operator types. The necessary number of computations is also affected by the sequence of events in any individual trajectory. The bound of Theorem \ref{thm:lptcomplexity} is the most general bound, but computation time may be much lower in practice.

\begin{theorem}[Live LTL Progress Tracking Complexity]\label{thm:lptcomplexity}
For a formula tree with height L, base node at level $l=0$, and maximum 2 children per node evaluated on trace $\rho$, complexity has upper bound
    \begin{equation}\label{complexity_lpt}
    \mathcal{O}(2^{L}|\rho|^2).
\end{equation}
\end{theorem}

\begin{proof}
    Consider a trace with from $T_0$ to $T_f$, i.e. $|\rho| = T_f-T_0$, for which an LPT update is performed at every time step. For each node $\varphi_{...x}$ in a tree, the corresponding module is called once per update; each module produces a tracking value for every $t_0$ in $\{T_0,...,t'\}$, resulting in $\leq t'-T_0$ evaluations per module call. At each $t_0$, the module performs \textbf{one evaluation per argument} $\varphi_{...x.y}$. Thus, any node with a binary operator (two arguments) performs $n\leq 2(t'-T_0)$ evaluations per update, yielding conservative bound $n<2|\rho|$.
    \par A tree with height L, base node at level $l=0$, and maximum 2 children per node will have $n_L \leq 2^L$ nodes at level L. In particular, since the final level $L$ must consist of atomic nodes, this may be further bounded as $n_L \leq 2^{L-1}$. Furthermore, atomic nodes require only \textbf{one} evaluation per call. This means that the number of evaluations for a tree of depth $L$ is of order $\mathcal{O}(2^{L-1}2|\rho|)$, or $\mathcal{O}(2^{L}|\rho|)$.
    \par Now, if an update is performed at every time step, there are $|\rho|$ total updates. The upper bound on complexity thus becomes $2^{L}|\rho||\rho|$, or
    \begin{equation*}
    \mathcal{O}(2^{L}|\rho|^2)
    \end{equation*}
    for the full tree. 
\end{proof}

\subsection{Automata-Theoretic Comparison}\label{automatoncompare}
In this section, we provide a comparison of our algorithm structure with the automata of LTL synthesis using an illustrative example. 
\begin{example}[LPT vs. LTL Synthesis: Automaton Perspective]\label{ex:automata}
Consider the LTL rule 
\begin{equation}
    \varphi = \mathbf{G}(\alpha\rightarrow\mathbf{X}\beta)
\end{equation}
where $\alpha,\beta$ are atomic propositions.  
\end{example}

 In this case, LTL synthesis represents $\varphi$ as the automaton shown in Figure \ref{fig:full_rule_automaton}. State transitions occur in the automaton at each time step in the trace $\rho$, where the domain of the transition function is the \textit{alphabet} $\Sigma=\{\epsilon, \{\alpha\}, \{\beta\}, \{\alpha,\beta\}\}$. Here, $\epsilon$ represents all sets of atomic propositions that do not contain $\alpha$ or $\beta$. 
 \par Importantly, when interpreting or tracking system progress, the automaton in Figure \ref{fig:full_rule_automaton} provides limited information. State $q_1$ is an accepting state; this means that, if the automaton transitions to $q_1$ at a given time step, the full rule $\phi$ is true on the current trace. State $q_3$ is a non-accepting sink state, so a transition here means that the rule must be false (and remain false for the rest of the trace). However, tracking only these states may provide extremely sparse information. Indeed, for our example rule, $q_1$ is not a sink, and so the system transitions away from it any time $\{\alpha\}, \{\alpha,\beta\}$ appear in the trace. Therefore, if $\varphi$ is true for the entire trace, we will not have this information until the final time step. Note that the other states $q_0, q_2$ are not as readily interpreted; they are abstract automaton states, constructs of the LTL synthesis. 
 
\begin{figure}[h]
\centering
\includegraphics[width=.4\columnwidth]{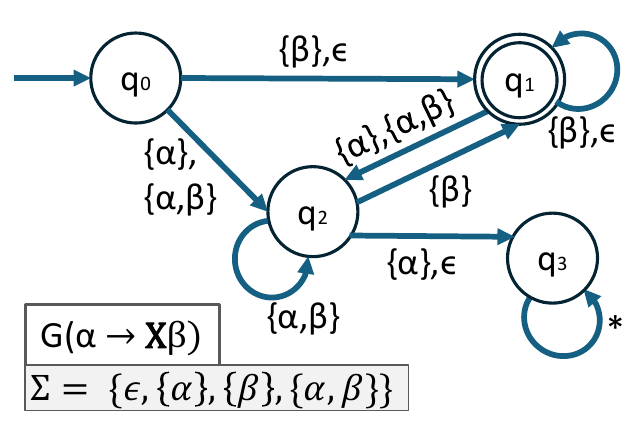}
\caption{}
\label{fig:full_rule_automaton}
\end{figure}

\par Now we consider the automaton interpretation of LPT. Recall that LPT assigns status values for \textit{every time step} in a trace for each rule argument. This is done modularly, where the module for argument $\varphi_x$ updates a vector $\vec{\xi}_x = (\xi[T_0],...,\xi[T_f])$. The atomic modules ingest trace time steps, giving them the same alphabet as before. For non-atomic arguments, the input to each module is the $\vec{\xi}$ of its child modules, meaning that their alphabet is
\begin{equation}
    \Sigma = \{(x_{T_0},...,x_{T_f})\ | \ x_i\in \{-1, 0, 1\} \}
\end{equation}
for unary operators, and
\begin{equation}
    \Sigma = \{((x_{T_0},...,x_{T_f}), (y_{T_0},...,y_{T_f}))\ | \ x_i, y_j\in \{-1, 0, 1\} \}
\end{equation}
for binary ones. For this rule, then, the LPT algorithm is equivalent to the series of automata shown in Figure \ref{fig:lpt_automata}. Every automaton has the \textbf{same structure}, with five transitions and the true (h) and false (v) values captured by sink states 0, 1. This matches the intuitive purpose of the status assignments, which should only assign $1,0$ on a trace segment once the value is conclusive for the entire future trace. The difference between module types is determined solely by which values from $\Sigma$ lead to each of the five possible transitions. 
\par For simplicity, we show automata to calculate the LPT status value for a single time step $t$, i.e. the tracking vector entries $\xi[t]$ for each node $(\mathbf{G}\phi), (\phi_1\rightarrow \phi_2), (\mathbf{X}\phi), (\alpha), (\beta).$ A key takeaway from this diagram is that values are available to track progress through every single argument, providing much more detail. Specifically, LPT provides one tracking state (with possible values true, false, open) per argument, yielding a 5-valued vector for each time step in this example (and up to $3^5$ distinct assignments). In contrast, the automaton of LPT synthesis assigns a single abstract state for the entire rule at each time step, where multiple possible values roughly correspond to `open' states; the vector of information allows for 4 distinct assignments.

\par In all, while the LTL synthesis approach results in a more concise representation, LPT provides higher-resolution tracking data. 

\begin{figure}[h]
\centering
\includegraphics[width=0.7\columnwidth]{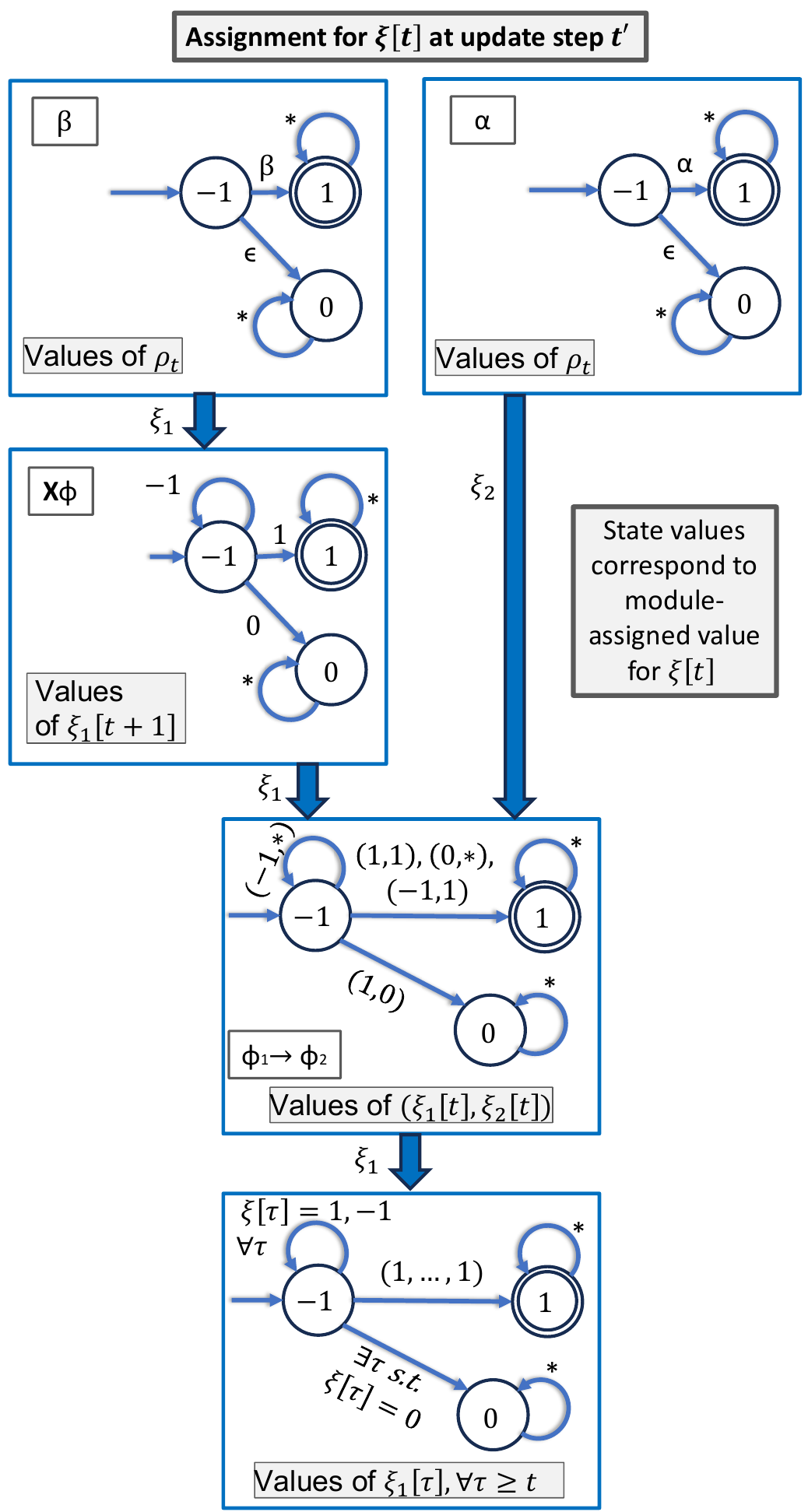}
\caption{LPT algorithm expressed as a series of automata for example formula $\mathbf{G}(\alpha\rightarrow\mathbf{X}\beta)$.}
\label{fig:lpt_automata}
\end{figure}
\FloatBarrier

\section{Conceptual Application: Reward Machines}
\begin{figure}[h]
\centering
\includegraphics[width=.7\columnwidth]{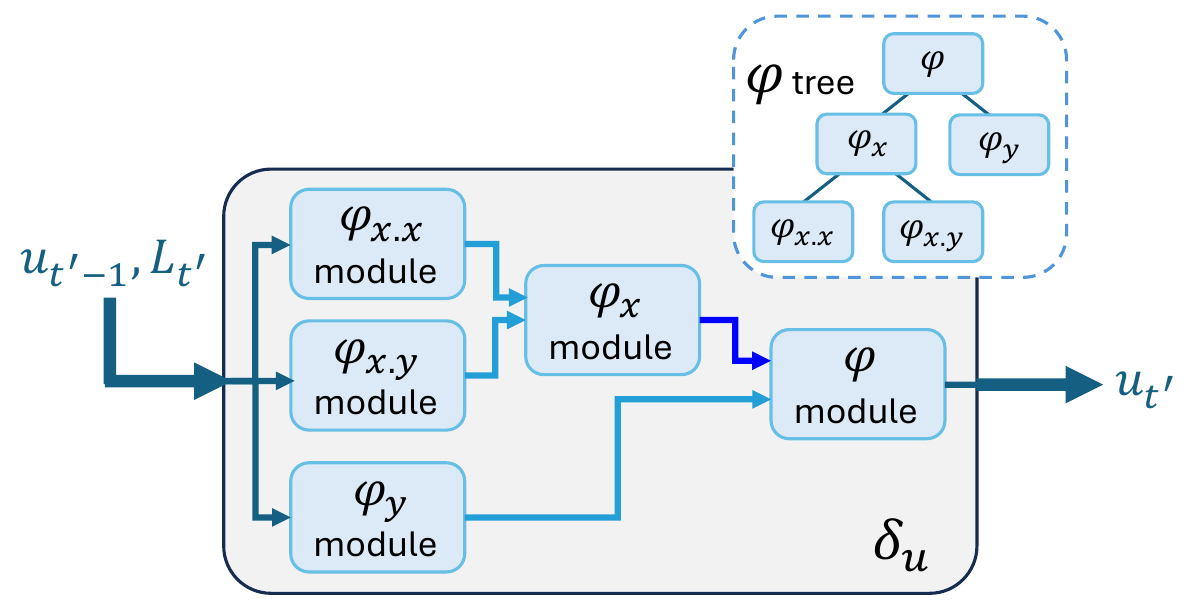}
\caption{Transition function $\delta_u$ of an LPT reward machine. The internal state is dependent on the formula $\varphi$. In this example, $u_{t'} = (\vec{\xi}, \vec{\xi}_x, \vec{\xi}_{x.x},\vec{\xi}_{x.y},\vec{\xi}_{y}, t')$.}
\label{fig:lpt_reward_machine}
\end{figure}

\par In the introduction, we suggest that LPT be applied to guide exploration and incentivize diverse behaviors. A starting point is the application of LPT in a reward machine. 

\par Taking tracking vectors as states and considering the dynamics by which they are updated, LPT fits easily into an RM framework. Recalling the Preliminaries, RMs serve as a replacement for simple reward functions in an MDP, assigning rewards based on a separate state space and transition function within the machine itself. 
\par Consider Figure \ref{fig:lpt_reward_machine}. As shown in the figure, we first define an RM-internal state $(u,L)$. Here, $u$ contains a current time as well as the tracking vectors for all nodes of a given specification. $L$ is a set of labels as before. Just as in LPT, we take the previous tracking vector $u_{t'-1}$ as an input alongside the most recent trace step $L_{t'}$. The updated tracking vector is then calculated modularly exactly as before, shown in the diagram as the output $u_{t'}$. Since the original LPT algorithm does not require any additional or intermediate inputs beyond $(u_{t'-1},L_{t'})$ as described, it may be treated as a self-contained transition function. In the diagram, we show it as $\delta_u$.

\par Given this internal state, a new reward function $\delta_r$ may be established. In the general definition of RMs, $\delta_r$ may map $U$ to a function space, i.e. a function $R:S\times A\rightarrow \mathbb{R}$ is output for each $u\in U$. 

\par Suppose all rewards are non-negative. Then, to incentivize a specific tracking state $u_{goal}$, for example, take 
\begin{equation}
    \delta_r = \begin{cases}
R(s,a) & u_{t'}=u_{goal}\\
0  &  \text{otherwise.}
\end{cases}
\end{equation}
Consider the common case that an agent has learned a specific behavior, but different behaviors are desired. Collecting the tracking vectors observed during the first behavior in the set $U_{old}$:
\begin{equation}
    \delta_r = \begin{cases}
R(s,a) & u_{t'}\not\in U_{old}\\
0  &  \text{otherwise.}
\end{cases}
\end{equation}

When applying an RM to an MDP, it is necessary to give the machine access to the RM state as well as its own, such that policies may have the form $\pi:S\times U\rightarrow A$. 

\section{Conclusion and Future Work}
In this paper, we establish Live Linear Temporal Logic Progress Tracking (LPT), an algorithm for detailed tracking of autonomous agent trajectories. LPT produces vector representations of an agent's progress through tasks that are specified in finite Linear Temporal Logic (LTL). These vectors encode behavioral information such as the order in which task-related events occur. By compressing vectors into a `signature', it is possible to compare trajectories in terms of their behavioral features, independently of specific timing or time scales.

\par We suggest that LPT has promising applications in reinforcement learning and give a framework for an LPT-based reward machine. By providing progress-based non-Markovian rewards, the machine could be used to incentivize incremental progress through tasks or exploration of diverse behaviors. Future work will apply the conceptual framework of LPT to RL, both for exploration and agent performance diagnostics.

\acks{This research was funded in part by an Army Educational Outreach Program fellowship.}

\bibliography{References}
\appendix
\section{Theorem \ref{correctnessLPT}}\label{appendixA}

 \textbf{Theorem \ref{correctnessLPT}}: \textit{Consider an LTL formula $\varphi$ and trace $\rho^{T_0...t'}$ which updates incrementally ($t'=T_0,...,T_f$). For inputs at each update $t'$, the LPT algorithm always assigns valid tracking value vectors $\vec{\xi}$ by Definition \ref{def:validtvv} for all nodes of $\varphi$.}

\par We prove the theorem in several steps. First, we will briefly confirm that the proof does indeed instantiate tracking value vectors for all nodes at each update. By the algorithm description, all modules take the following inputs:
\begin{itemize}
    \item \textbf{Initial and current times} $T_0, t'$, respectively
    \item \textbf{Trace} $\rho^{T_0...t'} = (L_{T_0},...,L_{t'})$
    \item \textbf{Tracking value vectors} $\vec{\xi}_x$ for each argument $\varphi_x$ of $\varphi$
\end{itemize}
Importantly, the LPT algorithm works up from child to parent nodes, processing all children of a parent before progressing to the parent. This means that tracking value vectors for the arguments of a formula will always be instantiated by the time the formula module is called, continuing to the full rule at the trunk of the tree.

\subsection{Foundation: Proof by Induction}
\par Critically for the proof, \textbf{all $\vec{\xi}$ are initialized with $\xi[t]=-1$ for all $t\in\{T_0,...,T_f\}$}. By Definition \ref{def:lpt_trackingvalvec}, this corresponds to an initial value of (o), and the vector is trivially valid by Definition \ref{def:validtvv}.
\par Then, to prove that each module produces a valid tracking value vector for its own node, it is possible to use induction: if all $\vec{\xi}$ at $t$ are initially valid by Definition \ref{def:validtvv}, a module which provably updates $\vec{\xi}$ according to Definition \ref{def:validtvv} will also produce a valid $\vec{\xi}$ at $t+1$. More specifically, the proof takes the following format:
\begin{enumerate}
    \item Show that $\vec{\xi}$ are initially valid tracking value vectors in line with Definition \ref{def:validtvv}.
    \item When performing an update at $t'$, assume that all $\vec{\xi}_{t'-1}$ are valid tracking value vectors.
    \item Prove that, given a valid $\vec{\xi}_{t'-1}$, the algorithm generates a $\vec{\xi}_{t'}$ which is valid.
\end{enumerate}
(1) is handled above. To complete (3), all updates to $\vec{\xi}$ must be examined. By the LPT algorithm, all updates take place either in the defined operator modules (Tables \ref{tab:logmodules_lpt} and \ref{tab:tempmodules_lpt}) or at the terminal update (equation \eqref{eq:terminal_eval1} in the LPT Algorithm). The proofs by module below show that each modular update from $t'$ to $t'+1$ produces valid tracking value vectors. The terminal update is handled in Lemma \ref{lemma:termupdate}. Once all updates are considered, it must follow that all $\vec{\xi}$ are valid tracking value vectors at all $t'$, proving Theorem \ref{correctnessLPT}.

\subsection{Select Definitions}
\par We revisit several relevant definitions here, including relevant notes on the proof which follows.
\par Concerning notation: in the evaluations of module compliance to Definition \ref{def:trackingprops},
\begin{itemize}
    \item by Definition \ref{def:lpt_trackingvalvec}, an equivalence is made between numerical values $1,0,-1$ as assigned by a module and their corresponding tracking values (h),(v),(o). In other words, if conditions are proven to satisfy the property for a given tracking value, it is treated as sufficient to show that the module assigns the matching numerical value for the appropriate entry of $\vec{\xi}$. This 
    \item Regarding notation, when an assignment is proven to satisfy Definition \ref{def:trackingprops} conditions for a tracking value, it is denoted by the corresponding value in bold, e.g. $\textbf{(h)}$ signifies that an assignment has been proven to meet the criterion for a time step to have value $1$.
\end{itemize}

\begin{tcolorbox}[colback=gray!10!white,colframe=white!80!black,title=\color{black}{Definition \ref{def:trackingprops}: Tracking value properties}]
Consider an LTL formula $\varphi$ on $\rho^{t...t'}$ with continuations $\rho^{t...t^+}$. Then a value assignment of (h),(v), or (o) is \textbf{valid} if:
\par $\bullet$ \textup{true (h)}: $\rho^{t...t^+}\models\varphi$ for all $t^+\geq t'$.
\par $\bullet$ \textup{false (v)}: $\rho^{t...t^+}\not\models\varphi$ for all $t^+\geq t'$.
\par $\bullet$ \textup{open (o)}: any of these conditions may or may not be met.
\end{tcolorbox}

\begin{tcolorbox}[colback=gray!10!white,colframe=white!80!black,title=\color{black}{Definition \ref{def:lpt_trackingvalvec}: Tracking value vector}]
 Consider $\varphi$ and finite traces in discrete time $\rho^{T_0...}$ with final time $T_f$. Then a tracking value vector for any argument of $\varphi$ can be expressed as
    \begin{equation}
        \vec{\xi} = (\xi[{T_0}], ...,\xi[{t'}],..., \xi[{T_f}])^T
    \end{equation}
     where $t'$ is a current update time with $t'\in\{T_0,...,T_f\}$. For all entries $\xi[t]$, 
    \begin{equation}
        \xi[t] = 
        \begin{cases}
            1 & \varphi \text{ has value (h) on }\rho^{t...t'}\\
            0 & \varphi \text{ has value (v) on }\rho^{t...t'}\\
            -1 & \text{otherwise}
        \end{cases}
    \end{equation}
\end{tcolorbox}

\par Now, for each module, Definition \ref{FiniteTraceDef} is extended as necessary to identify the conditions for $\rho^{t'...}\models\varphi$. 

\begin{tcolorbox}[colback=gray!10!white,colframe=white!80!black,title=\color{black}{Definition \ref{FiniteTraceDef} and Extensions}]
    LTL formula $\varphi$ is true on finite trace $\rho^{t_0...} = (L_{t_0},...,L_{T_f})$, denoted $\rho^{t_0...}\models\varphi$, \textbf{if} $T_0\leq t_0 \leq T_f$ \textbf{and}:
        \par$\bullet$ $\rho^{t_0...} \models \alpha$ where $\alpha\in P\quad$ iff $\quad\alpha\in L_0$
        \par$\bullet$ $\rho^{t_0...}\models \neg \varphi\quad$ iff $\quad\rho^{t_0...}\not\models\varphi$
        \par$\bullet$ $\rho^{t_0...} \models \varphi_1 \wedge \varphi_2\quad$ iff $\quad\rho^{t_0...}\models\varphi_1$ and $\rho^{t_0...}\models\varphi_2$
        \par and for the temporal operators \textup{Next} $\mathbf{X}$ and \textup{Until} $\mathbf{U}$, 
    \par$\bullet$ $\rho^{t_0...} \models \mathbf{X}\varphi\quad$ iff $\quad \rho^{t_0+1...}\models\varphi$ \textbf{and} $t_0<T_f$
    \par$\bullet$ $\rho^{t_0...}\models\varphi_1\mathbf{U}\varphi_2\quad$ iff $\quad \exists i\geq t_0$ s.t. $\rho^{i...}\models\varphi_2$ and $\rho^{k...} \models \varphi_1$ for all $k$ s.t. $t_0\leq k < i$ \textbf{and} $i\leq T_f$

Extending definitions for the additional operators yields

 \par $\bullet$ $\rho^{t_0...} \models \varphi_1\vee\varphi_2 \quad$ iff $\quad \rho^{t_0...}\models\varphi_1$ \textbf{or} $\rho^{t_0...}\models\varphi_2$

  \par $\bullet$ $\rho^{t_0...} \models \varphi_1\rightarrow\varphi_2 \quad$ iff $\quad\rho^{t_0...}\not\models\varphi_1$ \textbf{or} $\rho^{t_0...}\models\varphi_2$

  \par $\bullet$ $\rho^{t_0...} \models \mathbf{F}\varphi \quad$ iff $\quad \exists i\geq t_0$ s.t. $\rho^{i...}\models\varphi$ for $i\leq T_f$

 \par  $\bullet$ $\rho^{t_0...} \models \mathbf{G}\varphi \quad$ iff $\quad \rho^{i...}\models\varphi$ for all $i$ where $t_0\leq i\leq T_f$

 \par  $\bullet$ $\rho^{t_0...} \models \varphi_1\mathbf{W}\varphi_2 \quad$ iff $\quad (\exists i\geq t_0$ s.t. $\rho^{i...}\models\varphi_2$ and $\rho^{k...} \models \varphi_1$ for all $t_0\leq k < i$ \textbf{and} $i\leq T_f)$ \textbf{or} $(\rho^{i...}\models\varphi_1$ for all $i$ where $t_0\leq i\leq T_f)$
 
  \par $\bullet$ $\rho^{t_0...}\models\varphi_1 \mathbf{R} \varphi_2\quad$ iff $\quad$ for all $i$ s.t. $t_0\leq i\leq T_f$, $(\rho^{i...}\models\varphi_2)\vee (\exists k<i$ s.t. $\rho^{k...}\models\varphi_1)$

  \par $\bullet$ $\rho^{t_0...}\models\varphi_1 \mathbf{M} \varphi_2\quad$ iff $\quad\exists i,t_0\leq i\leq T_f$ s.t. $(\rho^{i...}\models \varphi_1) \wedge (\rho^{i...}\models\varphi_2) \wedge (\rho^{k...} \models \varphi_2$ for all $k$ s.t. $t_0\leq k < i)$
\end{tcolorbox}

\subsection{Lemmas and Properties of $\vec{\xi}$}
All valid $\vec{\xi}$ can be shown to have specific additional properties as a consequence of Definition \ref{def:trackingprops}. An important result is given in Lemma \ref{lemma:xiproperty}. Finally, the terminal update step in the LPT algorithm is handled by Lemma \ref{lemma:termupdate}.

\begin{lemma}[Dynamics of $\vec{\xi}$]\label{lemma:xiproperty}
    Consider $\vec{\xi}$ which is a valid tracking value vector by Definition \ref{def:validtvv} at update $t'$. Denote all entries at $t'$ by $\xi[t]_{t'}$. Then, for entries $\xi[t]_{t''}$ at a later update,
    \begin{align}
        (\xi[t]_{t'}&=1 \text{ valid})  \Rightarrow (\xi[t]_{t''}=1 \text{ valid})\\
        (\xi[t]_{t'}&=0 \text{ valid})  \Rightarrow (\xi[t]_{t''}=0 \text{ valid})
    \end{align}
    for all $t''>t'$.
    
\end{lemma}

\begin{proof}
Consider a valid tracking value vector $\vec{\xi}$ for some $\varphi$ and update time $t'$. First suppose $\xi[t]=1$ for some $t$, corresponding to (h). Then, by the conditions of Definition \ref{def:trackingprops},
\begin{equation}\label{eq:lemthing1}
\rho^{t...t^+}\models\varphi\quad\forall t^+\geq t'
\end{equation}
where $\rho^{t...t^+}$ is a continuation of $\rho^{t...t'}$ as in Definition \ref{def:continuation}. 
\par Now select any $t''$ such that $t'<t''\leq t^+$. Clearly, if \eqref{eq:lemthing1} is true, it must also be true that
\begin{equation}
\rho^{t...t^+}\models\varphi\quad\forall t^+\geq t''.
\end{equation}
This is the required property of a valid assignment of (h) at $t$ by Definition \ref{def:trackingprops}, making $\xi[t]_{t''}=1$ valid by Definition \ref{def:validtvv}. Thus, 
$$(\xi[t]_{t'}=1 \text{ valid})  \Rightarrow (\xi[t]_{t''}=1 \text{ valid})$$
for any $t''>t'$. 
\par Now suppose $\xi[t]=0$ for $\varphi, t'$, corresponding to assignment (v). Then, by Definition \ref{def:trackingprops},
\begin{equation}\label{eq:lemthing2}
\rho^{t...t^+}\not\models\varphi\quad\forall t^+\geq t'
\end{equation}
where $\rho^{t...t^+}$ is a continuation of $\rho^{t...t'}$. Again, select any later update $t''$ such that $t'<t''\leq t^+$. Then, given \eqref{eq:lemthing2}, it must also be true that
\begin{equation}
\rho^{t...t^+}\not\models\varphi\quad\forall t^+\geq t''.
\end{equation}
This is the property of a valid assignment of (v) at $t$. Thus $\xi[t]_{t''}=0$ is a valid tracking value vector, and so
$$(\xi[t]_{t'}=0 \text{ valid})  \Rightarrow (\xi[t]_{t''}=0 \text{ valid})$$
for any $t''>t'$, completing the proof.

\end{proof}

\par Lemma $\ref{lemma:xiproperty}$ yields a useful consequence for updates of $\vec{\xi}$: 

\begin{tcolorbox}[colback=gray!10!white,colframe=white!80!black,title=\color{black}{Property of valid $\vec{\xi}$}]
Given any entry $\xi[t]$ which is valid by Definition \ref{def:validtvv} at update $t'$, the entry \textbf{remains valid} at any later update $t''>t'$ if no evaluation is performed.
\end{tcolorbox}

\begin{proof}
    \par Trivially by Definitions \ref{def:validtvv} and \ref{def:trackingprops}, any assignment $\xi[t]=-1$ is always valid (if unspecific). Thus, only entries with $\xi[t]=0,1$ can be potentially be invalid and may \textbf{require} evaluation.
    \par The algorithm initializes all entries with $\xi[t]=-1$. As the algorithm \textbf{never} evaluates entries with $t> t'$ given strictly increasing update time $t'$, 
    $$\xi[t]=-1\quad \forall t>t'.$$ 
    Thus, entries for $\xi[t]$, $t>t'$ are uniformly valid at $t'$ and do not require evaluation. Now, for all entries $\xi[t]$ with $t\leq t'$, 
    \begin{itemize}
        \item if $\xi[t] = 0$ at the time of a previous update $t^-$, i.e. $\xi[t]_{t^-} = 0$ with $t^-<t'$, then $\xi[t]_{t'}$ must remain 0 by Lemma \ref{lemma:xiproperty}. Thus, no new evaluation is required for such $\xi[t]$ at any update $t'$.
        \item if $\xi[t] = 1$, the same principle applies.
        \item if $\xi[t] = -1$, the entry is valid trivially.
    \end{itemize}
    Therefore, any entry $\xi[t]$ which is a valid tracking value at update $t'$ \textbf{remains valid} at any later update $t''>t'$ if no evaluation is performed.
\end{proof}

\begin{lemma}[Terminal Update $T_f$]\label{lemma:termupdate}
    Consider $\vec{\xi}$ which is valid by Definition \ref{def:validtvv} at final update $T_f$. Denote entries of any $\vec{\xi}_{T_f}$ by $\xi[t]_{T_f}$. Then, if the following terminal update is performed for each tree node $\varphi$: 
    \begin{align}
        \xi[t]=
        \begin{cases}
            1&\rho^{t...T_f}\models\varphi\\
            0&\rho^{t...T_f}\not\models\varphi,
        \end{cases}
    \end{align}
    the resulting $\vec{\xi}_{T_f}$ is a valid tracking value vector.
\end{lemma}
\begin{proof}
\par Beginning with the assignment $\xi[t]_{T_f} = 1$: by Definition \ref{def:trackingprops}, $\varphi$ has the property associated with (h) at  update $t'=T_f$ if, for all $t^+\geq T_f$, 
$$\rho^{t...t^+}\models \varphi.$$
\par As $T_f$ is the terminal time, the only such $t^+$ is $T_f$ itself. Therefore, if $\rho^{t...T_f}\models\varphi$, then $\xi[t]_{T_f}=1$ meets the criterion for valid $\vec{\xi}$.
\par Next, for the assignment $\xi[t]_{T_f} = 0$: by Definiton \ref{def:trackingprops}, $\varphi$ has the property associated with (v) at  update $t'=T_f$ if, for all $t^+\geq T_f$, 
$$\rho^{t...t^+}\not\models \varphi.$$
By the same logic as (h), this is equivalent to the condition $\rho^{t...T_f}\not\models\varphi$ at update $T_f$. Therefore, $\xi[t]_{T_f}=0$ meets the criterion for valid $\vec{\xi}$ in this case.
    
\end{proof}

 \clearpage
\section{Logical modules}

This section proves that, given update $t'$ and valid tracking value vectors for each child node of some $\varphi$, the logical modules assign a valid tracking value vector for $\varphi$ at $t'$.

\subsection{Atomic proposition}

\begin{table}[h]
\centering
\begin{tabular}{lp{13.0 cm}}
\toprule
\textbf{Atomic}  $(AP)$ & LTL: $\varphi=\alpha$\\
 \midrule
Initialize: & Load $\vec{\xi}$ for $\varphi$.\\
If $\alpha\in L_{t'}$:&Set $\xi[t']= 1$.\\
Else: &Set $\xi[{t'}] = 0$.\\

\bottomrule\\
\end{tabular}
\caption{\label{tab:apmoduleLPT} Atomic proposition module definition.}
\end{table}
\FloatBarrier

\begin{tcolorbox}[colback=gray!10!white,colframe=white!80!black,title=\color{black}{Definition \ref{FiniteTraceDef}: Formula Evaluation on Finite Traces (AP)}]
 LTL formula $\varphi$ is true on finite trace $\rho^{t_0...} = (L_{t_0},...,L_{T_f})$, denoted $\rho^{t_0...}\models\varphi$, if $T_0\leq t_0 \leq T_f$ \textbf{and}:
        \par$\rho^{t_0...} \models \alpha$ where $\alpha\in P\quad$ \textbf{iff} $\quad\alpha\in L_0$ $(*)$\end{tcolorbox}

\begin{table}[h]
\centering
\begin{tabular}{p{16.0 cm}}
\textbf{Proof: AP Module}\\
\toprule
\textit{Module updates step-by-step:}\\
\toprule
(\textbf{Initialize:} Load $\vec{\xi}$ for $\varphi$ and $\vec{\xi}_1$ for child node $\varphi_1$.)\\
\textbf{If} $\alpha\in L_{t'}$: \textbf{Set} $\xi[t']= 1$.

\begin{itemize}
    \item By $(*)$, $\alpha\in L_{t'}$ implies that $\rho^{t'...t^+}\models\varphi$ for all $t^+\geq t'$. \textbf{(h)} ($\xi_{t'}[t']=1$)
\end{itemize}

\textbf{Else}: \textbf{Set} $\xi[t']= 0$.

\begin{itemize}
    \item By $(*)$, $\alpha\not\in L_{t'}$ implies that $\rho^{t'...t^+}\not\models\varphi$ for all $t^+\geq t'$. \textbf{(v)} ($\xi_{t'}[t']=0$)
\end{itemize}\\

\bottomrule
\end{tabular}
\end{table}
\clearpage

\subsection{Not}

\begin{table}[h]
\centering
\begin{tabular}{lp{13.0 cm}}
\toprule
\textbf{Not} $(neg)$& LTL: $\varphi=\neg\varphi_1$\\
\midrule
Initialize: & Load $\vec{\xi}$ for $\varphi$ and $\vec{\xi}_1$ for child node $\varphi_1$.\\
For $t=T_0,...,t'$:&If $\xi_1[t]=0$:\quad Set $\xi[t]= 1$.\\
    &Else if $\xi_1[t]=1$:\quad Set $\xi[{t}] = 0$.\\

\bottomrule\\
\end{tabular}
\caption{\label{tab:negmoduleLPT} Not module definition.}
\end{table}
\FloatBarrier

\begin{tcolorbox}[colback=gray!10!white,colframe=white!80!black,title=\color{black}{Definition \ref{FiniteTraceDef}: Formula Evaluation on Finite Traces (neg)}]
    LTL formula $\varphi$ is true on finite trace $\rho^{t_0...} = (L_{t_0},...,L_{T_f})$, denoted $\rho^{t_0...}\models\varphi$, if $T_0\leq t_0 \leq T_f$ \textbf{and}:
        \par $\rho^{t_0...}\models \neg \varphi\quad$ \textbf{iff} $\quad\rho^{t_0...}\not\models\varphi$ $(*)$
\end{tcolorbox}

\begin{table}[h]
\centering
\begin{tabular}{p{16.0 cm}}
\textbf{Proof: Not Module}\\

\toprule
\textit{Module updates step-by-step:}\\
\toprule
(\textbf{Initialize:} Load $\vec{\xi}$ for $\varphi$ and $\vec{\xi}_1$ for child node $\varphi_1$.)\\

\textbf{For} $t=T_0,...,t'$:\\
\\
\quad \textbf{If} $\xi_1[t]=0$: \textbf{Set} $\xi[t]= 1$.
\begin{itemize}
    \item By the loop from $T_0,...,t'$, $t$ always satisfies $t\leq t'$ for current update time $t'$.
    \item $\xi_1[t]=0$ at $t'$ means that $\rho^{t...t^+}\not\models\varphi_1$ for all $t^+\geq t'$.
    \item By $(*)$, this is equivalent to $\rho^{t...t^+}\models\varphi$ for all $t^+\geq t'$. \textbf{(h)} ($\xi_{t'}[t]$)
\end{itemize}

\quad \textbf{Else if} $\xi_1[t]=1$: \textbf{Set} $\xi[{t}] = 0$.
\begin{itemize}
    \item $\xi_1[t]=1$ at $t'$ means that $\rho^{t...t^+}\models\varphi_x$ for all $t^+\geq t'$.
    \item Then by $(*)$, $\rho^{t...t^+}\not\models\varphi$ for all $t^+\geq t'$. \textbf{(v)} ($\xi_{t'}[t]$)
\end{itemize}\\

\bottomrule
\end{tabular}
\end{table}
\clearpage

\subsection{And}

\begin{table}[h]
\centering
\begin{tabular}{lp{13.0 cm}}
\toprule
\textbf{And} $(and)$&LTL: $\varphi=\varphi_1\wedge\varphi_2$\\
\midrule
Initialize: & Load $\vec{\xi}$ for $\varphi$ and $\vec{\xi}_1,\vec{\xi}_2$ for child nodes $\varphi_1,\varphi_2$.\\
For $t=T_0,...,t'$:&If $\xi_1[t]=1$ and $\xi_2[t]=1$:\quad Set $\xi[t] = 1$. \\
    &Else if $\xi_1[t]=0$ or $\xi_2[t]=0$:\quad Set $\xi[t] = 0$.\\
\bottomrule\\
\end{tabular}
\caption{\label{tab:andmoduleLPT} And module definition.}
\end{table}
\FloatBarrier

\begin{tcolorbox}[colback=gray!10!white,colframe=white!80!black,title=\color{black}{Definition \ref{FiniteTraceDef}: Formula Evaluation on Finite Traces (and)}]
  LTL formula $\varphi$ is true on finite trace $\rho^{t_0...} = (L_{t_0},...,L_{T_f})$, denoted $\rho^{t_0...}\models\varphi$, if $T_0\leq t_0 \leq T_f$ \textbf{and}:
        \par$\rho^{t_0...} \models \varphi_1 \wedge \varphi_2\quad$ \textbf{iff} $\quad\rho^{t_0...}\models\varphi_1$ \textbf{and} $\rho^{t_0...}\models\varphi_2$ $(*)$
\end{tcolorbox}

\begin{table}[h]
\centering
\begin{tabular}{p{16.0 cm}}
\textbf{Proof: And Module}\\
\toprule
\textit{Module updates step-by-step:}\\
\toprule
(\textbf{Initialize:} Load $\vec{\xi}$ for $\varphi$ and $\vec{\xi}_1, \vec{\xi}_2$ for child nodes $\varphi_1,\varphi_2$.)\\

\textbf{For} $t=T_0,...,t'$:\\
\\
\quad \textbf{If} $\xi_1[t]=1$ \textbf{and} $\xi_2[t]=1$: \textbf{Set} $\xi[t]= 1$.
\begin{itemize}
    \item By the loop from $T_0,...,t'$, $t$ always satisfies $t\leq t'$ for current update time $t'$.
    \item $\xi_1[t], \xi_2[t]=1$ at current update time $t'$ means that $\rho^{t...t^+}\models\varphi_x$, $\forall x\in\{1,2\}$ for all $t^+\geq t'$.
    \item By $(*)$, $\varphi_1\wedge\varphi_2$ is true whenever $\rho^{t...t^+}\models\varphi_x$ $\forall x\in\{1,2\}$. 
    \item Since this is true for all $t^+\geq t'$, $\rho^{t...t^+}\models\varphi$ for all $t^+\geq t'$. \textbf{(h)} ($\xi_{t'}[t]$)
\end{itemize}

\quad \textbf{Else if} $\xi_1[t]=0$ \textbf{or} $\xi_2[t]=0$: \textbf{Set} $\xi[{t}] = 0$.
\begin{itemize}
    \item $\xi_x[t]=0$ at update $t'$ means that $\rho^{t...t^+}\not\models\varphi_x$, $x\in\{1,2\}$ for all $t^+\geq t'$.
    \item By $(*)$, $\varphi_1\wedge\varphi_2$ is not true whenever $\rho^{t...t^+}\not\models\varphi_x$ for some $x\in\{1,2\}$. 
    \item Since this is the case for all $t^+\geq t'$, $\rho^{t...t^+}\not\models\varphi$ for all $t^+\geq t'$. \textbf{(v)} ($\xi_{t'}[t]$)
\end{itemize}\\

\bottomrule
\end{tabular}
\end{table}
\clearpage


\subsection{Or}

\begin{table}[h]
\centering
\begin{tabular}{lp{13.0 cm}}
\toprule
\textbf{Or} $(or)$&LTL: $\varphi=\varphi_1\vee\varphi_2$\\
\midrule
Initialize: & Load $\vec{\xi}$ for $\varphi$ and $\vec{\xi}_1,\vec{\xi}_2$ for child nodes $\varphi_1,\varphi_2$.\\
For $t=T_0,...,t'$:&If $\xi_1[t]=1$ or $\xi_2[t]=1$:\quad Set $\xi[t] = 1$. \\
    &Else if $\xi_1[t]=0$ and $\xi_2[t]=0$:\quad Set $\xi[t] = 0$.\\

\bottomrule\\
\end{tabular}
\caption{\label{tab:ormoduleLPT} Or module definition.}
\end{table}
\FloatBarrier

\begin{tcolorbox}[colback=gray!10!white,colframe=white!80!black,title=\color{black}{Definition \ref{FiniteTraceDef}: Formula Evaluation on Finite Traces (or)}]
LTL formula $\varphi$ is true on finite trace $\rho^{t_0...} = (L_{t_0},...,L_{T_f})$, denoted $\rho^{t_0...}\models\varphi$, if $T_0\leq t_0 \leq T_f$ \textbf{and}:
 \par$\rho^{t_0...} \models \varphi_1\vee\varphi_2 \quad$ \textbf{where} $\quad \rho^{t_0...} \models \neg(\neg\varphi_1\wedge\neg\varphi_2).$
 \begin{itemize}
     \item The requirement $\rho^{t_0...}\models\neg\varphi_1\wedge\neg\varphi_2$ is equivalent to requiring $\rho^{t_0...}\not\models\varphi_1$ and $\rho^{t_0...}\not\models\varphi_2$ (by def. of $\wedge$ and $\neg$). This condition, $(\rho^{t_0...}\not\models\varphi_1$ and $\rho^{t_0...}\not\models\varphi_2)$ is false when $(\rho^{t_0...}\models\varphi_1)$ or $(\rho^{t_0...}\models\varphi_2)$ is true. The final condition is thus
 \end{itemize}
 \par $\rho^{t_0...} \models \varphi_1\vee\varphi_2 \quad$ \textbf{iff} $\quad \rho^{t_0...}\models\varphi_1$ \textbf{or} $\rho^{t_0...}\models\varphi_2$ $(*)$
\end{tcolorbox}

\begin{table}[h]
\centering
\begin{tabular}{p{16.0 cm}}
\textbf{Proof: Or Module}\\
\toprule
\textit{Module updates step-by-step:}\\
\toprule
(\textbf{Initialize:} Load $\vec{\xi}$ for $\varphi$ and $\vec{\xi}_1, \vec{\xi}_2$ for child nodes $\varphi_1,\varphi_2$.)\\

\textbf{For} $t=T_0,...,t'$:\\
\\
\quad \textbf{If} $\xi_1[t]=1$\textbf{ or} $\xi_2[t]=1$: \textbf{Set} $\xi[t]= 1$.
\begin{itemize}
    \item By the loop from $T_0,...,t'$, $t$ always satisfies $t\leq t'$ for current update time $t'$.
    \item $\xi_1[t]=1$ or $\xi_2[t]=1$ at current update time $t'$ means that $\rho^{t...t^+}\models\varphi_x$ for some $x\in\{1,2\}$ for all $t^+\geq t'$.
    \item By $(*)$, $\varphi_1\vee\varphi_2$ is true whenever $\rho^{t...t^+}\models\varphi_x$ for some $x\in\{1,2\}$. 
    \item Since this is true for all $t^+\geq t'$, $\rho^{t...t^+}\models\varphi$ for all $t^+\geq t'$. \textbf{(h)} ($\xi_{t'}[t]$)
\end{itemize}

\quad \textbf{Else if} $\xi_1[t]=0$ \textbf{and} $\xi_2[t]=0$: \textbf{Set} $\xi[{t}] = 0$.
\begin{itemize}
    \item $\xi_1[t], \xi_2[t]=0$ at update $t'$ means that $\rho^{t...t^+}\not\models\varphi_x$, $\forall x\in\{1,2\}$ for all $t^+\geq t'$.
    \item By $(*)$, $\varphi_1\vee\varphi_2$ is not true whenever $\rho^{t...t^+}\not\models\varphi_x$ $\forall x\in\{1,2\}$. 
    \item Since this is the case for all $t^+\geq t'$, $\rho^{t...t^+}\not\models\varphi$ for all $t^+\geq t'$. \textbf{(v)} ($\xi_{t'}[t]$)
\end{itemize}\\

\bottomrule
\end{tabular}
\end{table}
\clearpage

\subsection{Implication}

\begin{table}[h]
\centering
\begin{tabular}{lp{13.0 cm}}
\toprule
\textbf{Implication}  $(\rightarrow)$&LTL: $\varphi=\varphi_1\rightarrow\varphi_2$\\
\midrule
Initialize: & Load $\vec{\xi}$ for $\varphi$ and $\vec{\xi}_1,\vec{\xi}_2$ for child nodes $\varphi_1,\varphi_2$.\\
For $t=T_0,...,t'$:&If $\xi_1[t]=0$ or $\xi_2[t]=1$:\quad Set $\xi[t] = 1$. \\
    &Else if $\xi_1[t]=1$ and $\xi_2[t]=0$:\quad Set $\xi[t] = 0$.\\
\bottomrule\\
\end{tabular}
\caption{\label{tab:implmoduleLPT} Implication module definition.}
\end{table}
\FloatBarrier

\begin{tcolorbox}[colback=gray!10!white,colframe=white!80!black,title=\color{black}{Definition \ref{FiniteTraceDef}: Formula Evaluation on Finite Traces ($\rightarrow$)}]
LTL formula $\varphi$ is true on finite trace $\rho^{t_0...} = (L_{t_0},...,L_{T_f})$, denoted $\rho^{t_0...}\models\varphi$, if $T_0\leq t_0 \leq T_f$ \textbf{and}:
 \par$\rho^{t_0...} \models \varphi_1\rightarrow\varphi_2 \quad$ where $\quad \rho^{t_0...} \models \neg\varphi_1\vee\varphi_2$

\begin{itemize}
    \item  We see that $\rho^{t_0...} \models \neg\varphi_1\vee\varphi_2$ when $\rho^{t_0...}\models\neg\varphi_1$ or $\rho^{t_0...}\models\varphi_2$ (by def. of $\vee$). Thus,
\end{itemize}

  \par$\rho^{t_0...} \models \varphi_1\rightarrow\varphi_2 \quad$ \textbf{iff} $\quad\rho^{t_0...}\not\models\varphi_1$ \textbf{or} $\rho^{t_0...}\models\varphi_2$ $(*)$
\end{tcolorbox}

\begin{table}[h]
\centering
\begin{tabular}{p{16.0 cm}}
\textbf{Proof: Implication Module}\\
\toprule
\textit{Module updates step-by-step:}\\
\toprule
(\textbf{Initialize:} Load $\vec{\xi}$ for $\varphi$ and $\vec{\xi}_1,\vec{\xi}_2$ for child nodes $\varphi_1,\varphi_2$.)\\

\textbf{For} $t=T_0,...,t'$:\\
\\
\quad \textbf{If} $\xi_1[t]=0$ \textbf{or} $\xi_2[t]=1$: \textbf{Set} $\xi[t] = 1$.
\begin{itemize}
    \item By the loop from $T_0,...,t'$, $t$ always satisfies $t\leq t'$ for current update time $t'$.
    \item $\xi_1[t]=0$ at current update time $t'$ means that $\rho^{t...t^+}\not\models\varphi_1$ for all $t^+\geq t'$. By $(*)$, this is a sufficient condition for $\rho^{t...t^+}\models\varphi$.
    \item Since this is the case for all $t^+\geq t'$, $\rho^{t...t^+}\models\varphi$ for all $t^+\geq t'$. \textbf{(h)} ($\xi_{t'}[t]$)
    \item $\xi_2[t]=1$ at $t'$ means that $\rho^{t...t^+}\models\varphi_2$ for all $t^+\geq t'$. By $(*)$, this is another sufficient condition for $\rho^{t...t^+}\models\varphi$.
    \item Since this is the case for all $t^+\geq t'$, $\rho^{t...t^+}\models\varphi$ for all $t^+\geq t'$. \textbf{(h)} ($\xi_{t'}[t]$)
\end{itemize}

\quad \textbf{Else if} $\xi_1[t]=1$ \textbf{and} $\xi_2[t]=0$: \textbf{Set} $\xi[t] = 0$.
    \begin{itemize}
    \item $\xi_1[t]=1$ at current update time $t'$ means that $\rho^{t...t^+}\models\varphi_1$ for all $t^+\geq t'$. 
    \item $\xi_2[t]=0$ at $t'$ means that $\rho^{t...t^+}\not\models\varphi_2$ for all $t^+\geq t'$. 
    \item By $(*)$, this means that $\rho^{t...t^+}\not\models\varphi$ for all $t^+\geq t'$. \textbf{(v)} ($\xi_{t'}[t]$)
\end{itemize}\\
\bottomrule
\end{tabular}
\end{table}

\clearpage
\section{Temporal modules}

Likewise, this section proves that, given update $t'$ and valid tracking value vectors for each child node of some $\varphi$, the temporal modules assign a valid tracking value vector for $\varphi$ at $t'$.

\subsection{Next}

\begin{table}[h]
\centering
\begin{tabular}{lp{13.0 cm}}
\toprule
\textbf{Next}  $(X)$&LTL: $\varphi=\mathbf{X}\varphi_1$\\
\midrule
Initialize: & Load $\vec{\xi}$ for $\varphi$ and $\vec{\xi}_1$ for child node $\varphi_1$.\\
For $t=T_0,...,t'$:&If $\xi_1[t] = 1$, $t>T_0$:\quad Set $\xi[t-1] = 1$.\\
&Else if $\xi_1[t] = 0$, $t>T_0$\textbf{ or }$t'=T_f$: \quad Set $\xi[t-1] = 0$.\\
\bottomrule\\
\end{tabular}
\caption{\label{tab:XmoduleLPT} Next module definition.}
\end{table}
\FloatBarrier

\begin{tcolorbox}[colback=gray!10!white,colframe=white!80!black,title=\color{black}{Definition \ref{FiniteTraceDef}: Formula Evaluation on Finite Traces (X)}]
LTL formula $\varphi$ is true on finite trace $\rho^{t_0...} = (L_{t_0},...,L_{T_f})$, denoted $\rho^{t_0...}\models\varphi$, if $T_0\leq t_0 \leq T_f$ \textbf{and}:
 \par $\rho^{t_0...} \models \mathbf{X}\varphi\quad$ \textbf{iff} $\quad \rho^{t_0+1...}\models\varphi$ \textbf{and} $t_0<T_f$ $(*)$
\end{tcolorbox}

\begin{table}[h]
\centering
\begin{tabular}{p{16.0 cm}}
\textbf{Proof: Next Module}\\
\toprule
\textit{Module updates step-by-step:}\\
\toprule
(\textbf{Initialize:} Load $\vec{\xi}$ for $\varphi$ and $\vec{\xi}_1$ for child node $\varphi_1$.)\\

\textbf{For} $t=T_0,...,t'$:\\
\\
\quad \textbf{If} $\xi_1[t] = 1$, $t>T_0$: \textbf{Set} $\xi[t-1] = 1$.
\begin{itemize}
    \item By the loop from $T_0,...,t'$, $t$ always satisfies $t\leq t'$ for current update time $t'$.
    \item $\xi_1[t]=1$ at $t'$ means that $\rho^{t...t^+}\models\varphi_1$ for all $t^+\geq t'$.
    \item Taking $t_0=t-1$, this is equivalent to $\rho^{t-1...t^+}\models\varphi$ by $(*)$ for all $t^+\geq t'$. \textbf{(h)} ($\xi_{t'}[t-1]$)
\end{itemize}

\quad \textbf{Else if} $\xi_1[t] = 0$, $t>T_0$\textbf{ or }$t'=T_f$: \textbf{Set} $\xi[t-1] = 0$.
\begin{itemize}
    \item $\xi_1[t]=0$ at $t'$ means that $\rho^{t...t^+}\not\models\varphi_1$ for all $t^+\geq t'$.
    \item Then by $(*)$, $\rho^{t-1...t^+}\not\models\varphi$ for all $t^+\geq t'$. \textbf{(v)} ($\xi_{t'}[t]$)
\end{itemize}\\

\bottomrule
\end{tabular}
\end{table}

\clearpage


\subsection{Eventual}

\begin{table}[h]
\centering
\begin{tabular}{lp{13.0 cm}}
\toprule
\textbf{Eventual} $(F)$&LTL: $\varphi=\mathbf{F}\varphi_1$\\
\midrule
Initialize: & Load $\vec{\xi}$ for $\varphi$ and $\vec{\xi}_1$ for child node $\varphi_1$.\\
For $t=T_0,...,t'$:&If $\xi_1[t]=1$:\quad For all $t_0$ s.t. $T_0\leq t_0 \leq t$: set $\xi[t_0]= 1$.\\
\bottomrule\\
\end{tabular}
\caption{\label{tab:FmoduleLPT} Eventual module definition.}
\end{table}
\FloatBarrier

\begin{tcolorbox}[colback=gray!10!white,colframe=white!80!black,title=\color{black}{Definition \ref{FiniteTraceDef}: Formula Evaluation on Finite Traces (F)}]
LTL formula $\varphi$ is true on finite trace $\rho^{t_0...} = (L_{t_0},...,L_{T_f})$, denoted $\rho^{t_0...}\models\varphi$, if $T_0\leq t_0 \leq T_f$ \textbf{and}:
 \par $\rho^{t_0...} \models \mathbf{F}\varphi \quad$ where $\quad \rho^{t_0...} \models true\ \mathbf{U}\varphi$

\begin{itemize}
    \item By definition, $\rho^{t_0...}\models true\  \mathbf{U}\varphi_2\quad$ iff $\quad \exists i\geq t_0$ s.t. $\rho^{i...}\models\varphi_2$ and $\rho^{k...} \models true$ for all $t_0\leq k < i$ and $i\leq T_f$. Clearly $\rho^{k...} \models true$ trivially, and thus
\end{itemize}

  \par $\rho^{t_0...} \models \mathbf{F}\varphi \quad$ \textbf{iff} $\quad \exists i\geq t_0$ s.t. $\rho^{i...}\models\varphi$ for $i\leq T_f$ $(*)$

\end{tcolorbox}

\begin{table}[h]
\centering
\begin{tabular}{p{16.0 cm}}
\textbf{Proof: Eventual Module}\\
\toprule
\textit{Module updates step-by-step:}\\
\toprule
(\textbf{Initialize:} Load $\vec{\xi}$ for $\varphi$ and $\vec{\xi}_1$ for child node $\varphi_1$.)\\

\textbf{For} $t=T_0,...,t'$:\\
\\
\quad \textbf{If} $\xi_1[t]=1$: \textbf{For all} $t_0$ s.t. $T_0\leq t_0 \leq t$: \textbf{set} $\xi[t_0]= 1$.
\begin{itemize}
    \item By the outer loop with $t=T_0,...,t'$, $t$ always satisfies $t\leq t'$ for current update time $t'$.
    \item $\xi_1[t]=1$ at $t'$ means that $\rho^{t...t^+}\models\varphi_1$ for all $t^+\geq t'$.
    \item Taking $i=t$ and $t_0\leq t$, $(*)$ gives that $\rho^{t_0...t^+}\models\varphi$ for all $t^+\geq t'$. \textbf{(h)} ($\xi_{t'}[t_0]$)
\end{itemize}\\

\bottomrule
\end{tabular}
\end{table}
\clearpage

\subsection{Global}

\begin{table}[h]
\centering
\begin{tabular}{lp{13.0 cm}}
\toprule
\textbf{Global} $(G)$&LTL: $\varphi=\mathbf{G}\varphi_1$\\
\midrule
Initialize: & Load $\vec{\xi}$ for $\varphi$ and $\vec{\xi}_1$ for child node $\varphi_1$. Set $t=t'$.\\
While $t\geq T_0$:& If $\xi_1[t]=0$:\quad For all $t_0$ s.t. $T_0\leq t_0 \leq t$: set $\xi[t_0]= 0$. Set $t=T_0-1$.\\
& Else: set $t=t-1$.\\

\bottomrule\\
\end{tabular}
\caption{\label{tab:GmoduleLPT} Global module definition.}
\end{table}
\FloatBarrier

\begin{tcolorbox}[colback=gray!10!white,colframe=white!80!black,title=\color{black}{Definition \ref{FiniteTraceDef}: Formula Evaluation on Finite Traces (G)}]
LTL formula $\varphi$ is true on finite trace $\rho^{t_0...} = (L_{t_0},...,L_{T_f})$, denoted $\rho^{t_0...}\models\varphi$, if $T_0\leq t_0 \leq T_f$ \textbf{and}:
 \par$\rho^{t_0...} \models \mathbf{G}\varphi \quad$ where $\quad \rho^{t_0...} \models \neg\mathbf{F}\neg\varphi$

\begin{itemize}
    \item We have that $\rho^{t_0...}\models\neg\mathbf{F}\psi$ iff $\not\exists i\geq t_0$ s.t. $\rho^{i...}\models\varphi$ for $i\leq T_f$. If this is true, it must be true that $\rho^{i...}\not\models\psi$ for all $i$ where $t_0\leq i\leq T_f$. Equivalently, $\rho^{i...}\models\neg\psi$ for all $t_0\leq i\leq T_f$. 
    \item Taking $\psi = \neg\varphi,$
\end{itemize}
 
 \par $\rho^{t_0...} \models \mathbf{G}\varphi \quad$ \textbf{iff} $\quad \rho^{i...}\models\varphi$ for all $i$ where $t_0\leq i\leq T_f$ $(*)$

\end{tcolorbox}

\begin{table}[h]
\centering
\begin{tabular}{p{16.0 cm}}
\textbf{Proof: Global Module}\\
\toprule
\textit{Module updates step-by-step:}\\
\toprule
(\textbf{Initialize:} Load $\vec{\xi}$ for $\varphi$ and $\vec{\xi}_1$ for child node $\varphi_1$.)\\

\textbf{While} $t\geq T_0$:\\
\\
\quad\textbf{If} $\xi_1[t]=0$: \textbf{For all} $t_0$ s.t. $T_0\leq t_0 \leq t$: \textbf{set} $\xi[t_0]= 0$. \textbf{Set} $t=T_0-1$.
\begin{itemize}
    \item By the outer loop, $t$ begins at $t=t'$. 
    \item $\xi_1[t]=0$ at $t'$ means that $\rho^{t...t^+}\not\models\varphi_1$ for all $t^+\geq t'$.
    \item Take $i=t$. Then $(*)$ gives that $\rho^{t...t^+}\not\models\varphi$ for all $t^+\geq t'$. \textbf{(v)} ($\xi_{t'}[t]=0$)
    \item Moreover, the same selection of $i=t$ satisfies $t_0\leq i$ for all $t_0$ in $\{T_0,...,t\}$. Therefore, from $(^*)$, $\rho^{t_0...t^+}\not\models\varphi$ for all $t^+\geq t'$ \textbf{(v)} for every $\xi_{t'}[t_0]$ as well.
\end{itemize}\\
\quad \textbf{Else:} \textbf{set} $t=t-1$.
\begin{itemize}
    \item (No assignment of $\vec{\xi}$ is made here)
\end{itemize}\\

\bottomrule
\end{tabular}
\end{table}

\clearpage

\pagebreak
\subsection{Until}
\begin{algorithm}[h]
\SetAlgoLined
\SetKwFunction{getUpdates}{getUpdates}

\nl\textbf{Initialize:} Load $\vec{\xi}$ for $\varphi$ and $\vec{\xi}_1,\vec{\xi}_2$ for child nodes $\varphi_1,\varphi_2$.\\
\nl Set  $t_0,t= T_0$.

\nl \While{$t\leq t'$}{
    \nl\nllabel{lin:uline4}\uIf{$\xi_1[t]=0$ and $\xi_2[t]=0$}{
        \nl\nllabel{lin:uline5}Set $\xi[t] = 0$\\
        \nl\nllabel{lin:uline6}Set $t = t+1$\\
        \nl\nllabel{lin:uline7}Set $t_0 = t$
    }
    \nl\nllabel{lin:uline8}\uElseIf{$\xi_2[t] = 1$}{
        \nl\nllabel{lin:uline9}Set $\xi[t_0],...,\xi[t] = 1$\\
        \nl\nllabel{lin:uline10}Set $t = t+1$\\
        \nl\nllabel{lin:uline11}Set $t_0 = t$
    }
    \nl\nllabel{lin:uline12}\uElseIf{$\xi_1[t] = 0$ or $\xi_1[t] = -1$}{
        \nl\nllabel{lin:uline13}Set $t = t+1$\\
        \nl\nllabel{lin:uline14}Set $t_0 = t$
    }
    \nl\nllabel{lin:uline15}\uElse{
        \nl\nllabel{lin:uline16}Set $t = t+1$
    }
}
\caption{Until Module for LPT.}
\end{algorithm}
\FloatBarrier

\begin{tcolorbox}[colback=gray!10!white,colframe=white!80!black,title=\color{black}{Definition \ref{FiniteTraceDef}: Formula Evaluation on Finite Traces (U)}]
LTL formula $\varphi$ is true on finite trace $\rho^{t_0...} = (L_{t_0},...,L_{T_f})$, denoted $\rho^{t_0...}\models\varphi$, if $T_0\leq t_0 \leq T_f$ \textbf{and}:
 \par$\rho^{t_0...}\models\varphi_1\mathbf{U}\varphi_2\quad$ \textbf{iff} $\quad \exists i\geq t_0$ s.t. $\rho^{i...}\models\varphi_2$ \textbf{and} $\rho^{k...} \models \varphi_1$ for all $k$ s.t. $t_0\leq k < i$ \textbf{and} $i\leq T_f$ $(*)$

\end{tcolorbox}

\begin{table}[h]
\centering
\begin{tabular}{p{16.0 cm}}
\textbf{Proof: Until Module}\\
\toprule
\textit{Module updates step-by-step:}\\
\toprule
(\textbf{Initialize:}  Load $\vec{\xi}$ for $\varphi$ and $\vec{\xi}_1,\vec{\xi}_2$ for child nodes $\varphi_1,\varphi_2$.
\textbf{Set} $t_0,t=T_0$.)\\

\textbf{While} $t\leq t'$:\\
\\
\quad \textbf{If $\xi_1[t]=0$ and $\xi_2[t]=0$: Set $\xi[t]=0$. Set $t=t+1$. Set $t_0=t$.} (Lines \ref{lin:uline4}-\ref{lin:uline7})
\begin{itemize}
    \item Consider $(*)$, taking $k=t$. Then $\xi_1[k]=0$ at $t'$ means that $\rho^{k...t^+}\not\models\varphi_1$ for all $t^+\geq t'$. 
    \item For any $i$ at which $\rho^{i...t^+}\models\varphi_2$, it is thus clearly not true that $\rho^{k...t^+}\models\varphi_1$ for all $t_0\leq k < i$ for any choice of $t_0\leq k$ \textbf{(v)}.
\end{itemize}

\quad \textbf{Else if $\xi_2[t]=1$: Set $\xi[t_0],...,\xi[t]=1$. Set $t=t+1$. Set $t_0=t$.} (Lines \ref{lin:uline8}-\ref{lin:uline11})
\begin{itemize}
    \item Consider $(*)$ again, taking $i=t$. Since $\xi_2[i]=1$ at $t'$ means that $\rho^{i...t^+}\models\varphi_2$ for all $t^+\geq t'$, the first condition of $(*)$ is met. 
    \item The remaining condition requires that $\rho^{k...t^+}\models\varphi_1$ for all $t_0\leq k < i$.
    \item Suppose this is uncertain for some $k$. Then $\xi_1[k]\neq 1$. However, by the module, whenever $\xi_1[k]\neq 1$, $t_0$ is set so that $t_0=t$.
    \item As $t$ is strictly increasing, this means either $t_0=t$ or $\xi_1[k]=1$ $\forall k$ s.t. $t_0\leq k < i$. Thus, the conditions of $(*)$ are indeed met, and the assignment $\xi[k]= 1$ must be correct for all $k$ in $\{t_0,...,t\}$. \textbf{(h)}
\end{itemize}

\quad\textbf{Else if $\xi_1[t] = 0$ or $\xi_1[t] = -1$: Set $t=t+1$. Set $t_0=t$} (Lines \ref{lin:uline12}-\ref{lin:uline14})
\begin{itemize}
    \item (No assignment of $\vec{\xi}$ is made here)
\end{itemize}

\quad\textbf{Else: Set $t=t+1$.} (Lines \ref{lin:uline15}-\ref{lin:uline16})
\begin{itemize}
    \item (No assignment of $\vec{\xi}$ is made here)
\end{itemize}\\

\bottomrule
\end{tabular}
\end{table}
\clearpage

\pagebreak
\subsection{Weak until}

\begin{tcolorbox}[colback=gray!10!white,colframe=white!80!black,title=\color{black}{Definition \ref{FiniteTraceDef}: Formula Evaluation on Finite Traces (W)}]
LTL formula $\varphi$ is true on finite trace $\rho^{t_0...} = (L_{t_0},...,L_{T_f})$, denoted $\rho^{t_0...}\models\varphi$, if $T_0\leq t_0 \leq T_f$ \textbf{and}:
 \par$\rho^{t_0...} \models \varphi_1\mathbf{W}\varphi_2 \quad$ where $\quad \rho^{t_0...} \models (\varphi_1\mathbf{U}\varphi_2)\vee\mathbf{G}\varphi_1$

\begin{itemize}
    \item We have that $\rho^{t_0...}\models (\varphi_1\mathbf{U}\varphi_2) \vee \mathbf{G}\varphi_1$ if $\rho^{t_0...} \models \mathbf{G}\varphi_1$ \textbf{or} $\rho^{t_0...}\models  \varphi_1\mathbf{U}\varphi_2$. The latter condition is true iff $\exists i\geq t_0$ s.t. $\rho^{i...}\models\varphi_2$ and $\rho^{k...} \models \varphi_1$ for all $t_0\leq k < i$ and $i\leq T_f$ (by def. of $\mathbf{U}$). Thus,
\end{itemize}

 \par $\rho^{t_0...} \models \varphi_1\mathbf{W}\varphi_2 \quad$ \textbf{iff }$\quad (\exists i\geq t_0$ s.t. $\rho^{i...}\models\varphi_2$ \textbf{and} $\rho^{k...} \models \varphi_1$ for all $t_0\leq k < i$ \textbf{and} $i\leq T_f)$ $(*)$ \textbf{or} $(\rho^{i...}\models\mathbf{G}\varphi_1$ for all $i$ where $t_0\leq i\leq T_f)$ $(**)$

\end{tcolorbox}

\begin{proof}
The proof for the Weak Until module is identical to that for the Until module above. This is a consequence of the fact that the condition $(*)$ above is the same as that for Until. 
\par The two operators are distinct only in their handling of the terminal time step $T_f$ of $\rho^{t_0...T_f}$; in particular, for Weak Until,
$$(\rho^{t_0...T_f}\models\mathbf{G}\varphi_1)\Rightarrow (\rho^{t_0...T_f}\models\varphi)$$
which means that $\rho^{T_f...T_f}\models\varphi_2$ is not necessary for satisfaction. This is different from Until, for which $(*)$ is the only satisfying condition.
\par The difference in requirement is handled by the \textbf{terminal update} step of the LPT algorithm, reproduced here:
\par \textbf{If} $t'=T_f$, perform \textbf{terminal evaluation} for all tree nodes $\phi$ of $\varphi$. For all $t\in\{T_0,...,T_f\}$, if $\xi[t]=-1$, set
    \begin{align}\label{eq:terminal_eval1}
        \xi[t]=
        \begin{cases}
            1&\rho^{t...T_f}\models\phi\\
            0&\rho^{t...T_f}\not\models\phi
        \end{cases}
    \end{align}
By this update, any $\xi[t] = -1$ are assigned appropriate truth values by the definition for $\phi$ on the complete trace; this is trivially possible, since the entire trace is known at $t'=T_f$, and equates to traditional LTL evaluation. This update will affect the only difference in evaluation of $\vec{\xi}$ between $\mathbf{W}$ and $\mathbf{U}$. If $\rho^{t...T_f}$ does not meet the stricter requirement, this last update sets $\xi[t] = 0$ for $\mathbf{U}$, but $\xi[t] = 1$ for $\mathbf{W}$.
\end{proof}
\clearpage

\pagebreak
\subsection{Strong release}

\begin{algorithm}[h]
\SetAlgoLined
\SetKwFunction{getUpdates}{getUpdates}

\nl\textbf{Initialize:} Load $\vec{\xi}$ for $\varphi$ and $\vec{\xi}_1,\vec{\xi}_2$ for child nodes $\varphi_1,\varphi_2$.\\
\nl Set  $t_0,t= T_0$.

\nl \While{$t\leq t'$}{
    \nl\nllabel{lin:mline4}\uIf{$\xi_[t]=0$}{
        \nl\nllabel{lin:mline5}Set $\xi[t] = 0$\\
        \nl\nllabel{lin:mline6}Set $t = t+1$\\
        \nl\nllabel{lin:mline7}Set $t_0 = t$
    }
    \nl\nllabel{lin:mline8}\uElseIf{$\xi_1[t] = 1$ and $\xi_2[t] = 1$}{
        \nl\nllabel{lin:mline9}Set $\xi[t_0],...,\xi[t] = 1$\\
        \nl\nllabel{lin:mline10}Set $t = t+1$\\
        \nl\nllabel{lin:mline11}Set $t_0 = t$
    }
    \nl\nllabel{lin:mline12}\uElseIf{$\xi_2[t] = -1$}{
        \nl\nllabel{lin:mline13}Set $t = t+1$\\
        \nl\nllabel{lin:mline14}Set $t_0 = t$
    }
    \nl\nllabel{lin:mline15}\uElse{
        \nl\nllabel{lin:mline16}Set $t = t+1$
    }
}
\caption{Release/Strong Release Module for LPT.}
\end{algorithm}

\begin{tcolorbox}[colback=gray!10!white,colframe=white!80!black,title=\color{black}{Definition \ref{FiniteTraceDef}: Formula Evaluation on Finite Traces (M)}]
LTL formula $\varphi$ is true on finite trace $\rho^{t_0...} = (L_{t_0},...,L_{T_f})$, denoted $\rho^{t_0...}\models\varphi$, if $T_0\leq t_0 \leq T_f$ \textbf{and}:
 \par $\rho^{t_0...} \models \varphi_1\mathbf{M}\varphi_2 \quad$ where $\quad \rho^{t_0...} \models \varphi_2 \mathbf{U}(\varphi_1\wedge\varphi_2)$

\begin{itemize}
    \item We have that $\rho^{t_0...}\models \varphi_2 \mathbf{U}(\varphi_1\wedge\varphi_2)$ iff $\exists i,t_0\leq i\leq T_f$ s.t. $\rho^{i...}\models(\varphi_1 \wedge \varphi_2)$ and $\rho^{k...} \models \varphi_2$ for all $k$ s.t. $t_0\leq k < i)$ (by def. of $\mathbf{U}$). We can rewrite this:
\end{itemize}
 
 \par $\rho^{t_0...}\models\varphi_1 \mathbf{M} \varphi_2\quad$ \textbf{iff} $\quad\exists i,t_0\leq i\leq T_f$ s.t. $(\rho^{i...}\models \varphi_1) \wedge (\rho^{k...} \models \varphi_2$ for all $k$ s.t. $t_0\leq k \leq i)$ $(*)$

\end{tcolorbox}

\begin{table}[h]
\centering
\begin{tabular}{p{16.0 cm}}
\textbf{Proof: Strong Release Module}\\
\toprule
\textit{Module updates step-by-step:}\\
\toprule
(\textbf{Initialize:}  Load $\vec{\xi}$ for $\varphi$ and $\vec{\xi}_1,\vec{\xi}_2$ for child nodes $\varphi_1,\varphi_2$.
\textbf{Set} $t_0,t=T_0$.)\\

\textbf{While} $t\leq t'$:\\
\\
\quad \textbf{If $\xi_2[t]=0$: Set $\xi[t]=0$. Set $t=t+1$. Set $t_0=t$.} (Lines \ref{lin:mline4}-\ref{lin:mline7})
\begin{itemize}
    \item Consider $(*)$, taking $i=t$. Then $\xi_2[i]=0$ at $t'$ means that $\rho^{i...t^+}\not\models\varphi_1$ for all $t^+\geq t'$. 
    \item This violates the second condition of $(*)$, which requires that $\rho^{i...t^+}\not\models\varphi_1$ $\forall k, t_0\leq k \leq i$ (where $t_0$ is an arbitrary value $t_0\leq k$). \textbf{(v)} 
\end{itemize}

\quad \textbf{Else if $\xi_1[t]=1$ and $\xi_2[t]=1$: Set $\xi[t_0],...,\xi[t]=1$. Set $t=t+1$. Set $t_0=t$.} (Lines \ref{lin:mline8}-\ref{lin:mline11})
\begin{itemize}
    \item Consider $(*)$ again, taking $i=t$. Since $\xi_1[i]=1$ at $t'$ means that $\rho^{i...t^+}\models\varphi_1$ for all $t^+\geq t'$, one condition of $(*)$ is met for that $i$. 
    \item The remaining condition requires that $\rho^{k...t^+}\models\varphi_2$ for all $t_0\leq k \leq i$.
    \item Suppose this is uncertain for some $k$. Then $\xi_2[k]\neq 1$. However, by the module, whenever $\xi_1[k]\neq 1$, $t_0$ is set so that $t_0=t$.
    \item As $t$ is strictly increasing, this means either $t_0=t$ or $\xi_2[k]=1$ $\forall k$ s.t. $t_0\leq k \leq i$. Thus, the second condition of $(*)$ is indeed met, and the assignment $\xi[k]= 1$ must be correct for all $k$ in $\{t_0,...,t\}$. \textbf{(h)}
\end{itemize}

\quad\textbf{Else if $\xi_2[t] = -1$: Set $t=t+1$. Set $t_0=t$} (Lines \ref{lin:mline12}-\ref{lin:mline14})
\begin{itemize}
    \item (No assignment of $\vec{\xi}$ is made here)
\end{itemize}

\quad\textbf{Else: Set $t=t+1$.} (Lines \ref{lin:mline15}-\ref{lin:mline16})
\begin{itemize}
    \item (No assignment of $\vec{\xi}$ is made here)
\end{itemize}\\

\bottomrule
\end{tabular}
\end{table}
\clearpage

\pagebreak
\subsection{Release}

\begin{tcolorbox}[colback=gray!10!white,colframe=white!80!black,title=\color{black}{Definition \ref{FiniteTraceDef}: Formula Evaluation on Finite Traces (R)}]
LTL formula $\varphi$ is true on finite trace $\rho^{t_0...} = (L_{t_0},...,L_{T_f})$, denoted $\rho^{t_0...}\models\varphi$, if $T_0\leq t_0 \leq T_f$ \textbf{and}:
 \par$\rho^{t_0...} \models \varphi_1\mathbf{R}\varphi_2 \quad$ where $\quad \rho^{t_0...} \models \neg(\neg\varphi_1\mathbf{U}\neg\varphi_2)$

 \begin{itemize}
     \item Consider $\neg(\psi_1\mathbf{U}\psi_2)$. This is true on $\rho^{t_0...}$ iff $\neg(\exists i\geq t_0$ s.t. $\rho^{i...}\models\psi_2$ and $\rho^{k...} \models \psi_1$ for all $k$ s.t. $t_0\leq k < i)$ (by def. of $\mathbf{U}$) and $i\leq T_f$. The former is equivalent to requiring $\not\exists i\geq t_0$ s.t. ($\rho^{i...}\models\psi_2$ and $\rho^{k...} \models \psi_1$ for all $i,k$ where $t_0\leq k < i$). 
     \item This means that for all $i$ with $\rho^{i...}\models\psi_2$ (if such $i$ exist), there exists some $k$ with $t_0\leq k < i$ such that $\rho^{k...}\not\models\psi_1$. More formally, $\neg(\psi_1\mathbf{U}\psi_2)$ iff for all $i$ s.t. $t_0\leq i\leq T_f$, $\neg(\rho^{i...}\models\psi_2)\vee (\exists k<i$ s.t. $\rho^{k...}\not\models\psi_1)$.
     \item Substitute $\psi_1 := \neg\varphi_1$ and $\psi_2 := \neg\varphi_2$. Then $\rho^{t_0...}\models\neg(\neg\varphi_1 \mathbf{U} \neg\varphi_2)$ iff for all $i$ s.t. $t_0\leq i\leq T_f$, $\neg(\rho^{i...}\not\models\varphi_2)\vee (\exists k<i$ s.t. $\rho^{k...}\models\varphi_1)$, which yields
 \end{itemize}

 \par$\rho^{t_0...}\models\varphi_1 \mathbf{R} \varphi_2\quad$ \textbf{iff} $\quad$ for all $i$ s.t. $t_0\leq i\leq T_f$, $(\rho^{i...}\models\varphi_2)$ $(*)$ or $ (\exists k<i$ s.t. $\rho^{k...}\models\varphi_1)$ $(**)$

\end{tcolorbox}

\begin{proof} 
The proof for the Release module is identical to that for the Strong Release module above. The difference in conditions between these two operators is expressed in condition $(*)$ above; Release allows

$$(\rho^{i...T_f}\models\varphi_2\quad\forall i, t_0\leq i\leq T_f)\Rightarrow (\rho^{t_0...T_f}\models\varphi).$$

For Strong Release, there is a stricter requirement:
$$\rho^{i...}\models\varphi_1$$
for some $i$ such that $t_0\leq i$. Once again, the difference in requirement is handled by the \textbf{terminal update} step of the LPT algorithm.
This update will assign any $\xi[t_0] = -1$ where $\rho^{t_0...T_f}$ does not meet the stricter requirement to (v) for $\mathbf{M}$, but (h) for $\mathbf{R}$.
\end{proof}

\end{document}